\documentclass[letterpaper]{article} 
\usepackage{aaai24}  
\usepackage{times}  
\usepackage{helvet}  
\usepackage{courier}  
\usepackage[hyphens]{url}  
\usepackage{graphicx} 
\usepackage{natbib}  
\usepackage{caption} 
\usepackage{algorithm}
\usepackage{algorithmic}

\usepackage{newfloat}
\usepackage{listings}
\DeclareCaptionStyle{ruled}{labelfont=normalfont,labelsep=colon,strut=off} 
\floatstyle{ruled}
\newfloat{listing}{tb}{lst}{}
\floatname{listing}{Listing}
\usepackage{multirow}
\usepackage{makecell}
\usepackage{booktabs}
\usepackage{amsmath}
\usepackage{amssymb}
\usepackage{mathtools}
\usepackage{amsthm}
\usepackage{babel}
\usepackage{xcolor}         
\usepackage{dsfont}

\newcommand{\ie}{\text{i.e.}, }

\DeclarePairedDelimiter\br{(}{)}
\DeclarePairedDelimiter\brc{\{}{\}}
\DeclarePairedDelimiter\abs{\lvert}{\rvert}
\DeclarePairedDelimiter\norm{\lVert}{\rVert}
\DeclarePairedDelimiter\innorm{\langle}{\rangle}

\DeclareMathOperator{\R}{\mathbb{R}} 
\DeclareMathOperator*{\argmax}{arg\,max}
\DeclareMathOperator*{\argmin}{arg\,min}

\DeclareMathOperator{\St}{\mathcal{S}}
\DeclareMathOperator{\A}{\mathcal{A}}
\DeclareMathOperator{\X}{\mathcal{X}}

\newcommand{\Rc}{\mathcal{R}}

\newcommand{\rop}{_{R}}

\newtheorem*{theorem*}{Theorem} 
\newtheorem*{corollary*}{Corollary}
\newtheorem*{lemma*}{Lemma}
\newtheorem*{proposition*}{Proposition}

\newtheorem{theorem}{Theorem}[]
\newtheorem{proposition}[theorem]{Proposition}
\newtheorem{lemma}[theorem]{Lemma}
\newtheorem{corollary}[theorem]{Corollary}

\newtheorem{remark}[theorem]{Remark}

\def\showComments{} 

\ifdefined\showComments
    \newcommand{\comU}[1]{\textcolor{red}{\{Uri: #1\}}}
    \newcommand{\comS}[1]{\textcolor{blue}{\{Shie: #1\}}}
    \newcommand{\comN}[1]{\textcolor{purple}{\{Navdeep: #1\}}}
    \newcommand{\comE}[1]{\textcolor{orange}{\{Esther: #1\}}}
\else
    \newcommand{\comU}[1]{}
    \newcommand{\comS}[1]{}
    \newcommand{\comN}[1]{}
    \newcommand{\comE}[1]{}
\fi

\title{Solving Non-rectangular Reward-Robust MDPs via Frequency Regularization}

\author {
    Uri Gadot\textsuperscript{\rm 1},
    Esther Derman\textsuperscript{\rm 2},
    Navdeep Kumar\textsuperscript{\rm 1},
    Maxence Elfatihi\textsuperscript{\rm 4},
    Kfir Levy\textsuperscript{\rm 1},
    Shie Mannor\textsuperscript{\rm 1,3}
}

\affiliations{
    \textsuperscript{\rm 1}Technion - Israel Institute of Technology\\
    \textsuperscript{\rm 2}MILA, Université de Montréal\\
    \textsuperscript{\rm 3}NVIDIA Research\\
    \textsuperscript{\rm 4}IMT Atlantique\\
    ugdaot@gmail.com, esther.derman@mila.quebec, navdeepkumar@campus.technion.ac.il}

\begin{document}
\maketitle
\begin{abstract}
In robust Markov decision processes (RMDPs), it is assumed that the reward and the transition dynamics lie in a given uncertainty set. By targeting maximal return under the most adversarial model from that set, RMDPs address performance sensitivity to misspecified environments. Yet, to preserve computational tractability, the uncertainty set is traditionally independently structured for each state. This so-called rectangularity condition is solely motivated by computational concerns. As a result, it lacks a practical incentive and may lead to overly conservative behavior.
In this work, we study coupled reward RMDPs where the transition kernel is fixed, but the reward function lies within an $\alpha$-radius from a nominal one. We draw a direct connection between this type of non-rectangular reward-RMDPs and applying policy visitation frequency regularization. We introduce a policy-gradient method and prove its convergence. Numerical experiments illustrate the learned policy's robustness and its less conservative behavior when compared to rectangular uncertainty.
\end{abstract}

\section*{Introduction}
The Markov decision process (MDP) framework formalizes sequential decision-making problems where the goal is to find a policy that maximizes the agent's performance in a particular environment \cite{Sutton1998, puterman2014markov}. In most scenarios, the environment's dynamics and/or the reward function are partially known, perturbed by noise, or attacked in an adversarial way. For example, considering a self-driving car simulator, the discrepancy between the idealized virtual environment and unexpectedly varying weather, traffic, and road conditions raises significant challenges during training. Ignoring such model uncertainty can have detrimental effects on the agent's performance, potentially leading to catastrophic failure \cite{BiasVarianceShie}.


On the other hand, solving RMDPs with general uncertainty sets is known to be NP-hard \cite{wiesemann2013robust}. To address this issue, previous studies have focused on identifying sub-classes of coupled RMDPs that are still solvable in polynomial time \cite{k-rectangularRMDP, goyal2023robust}. Yet, the above studies have mostly focused on RMDPs with a known reward model but uncertain dynamics. Hence, little attention has been given to RMDPs with coupled reward uncertainty and known transition. 

Even when the model is comprehensively understood, the challenge of obtaining a precise reward function persists in many practical applications. This predicament can arise when employing a reward model trained on a subset of labeled data or when learning relies on human feedback or preferences. Additionally, although allowing ambiguity on the reward only can seem restrictive, it models a large class of sequential decision-making problems, including MDPs with deterministic transitions such as path planning. Consider again our self-driving car example and assume that its policy is deployed on real road conditions to drive towards a destination point. In this setting, not accounting for reward uncertainty during training could lead the car toward a different destination. On the other hand, a robust policy under rectangular reward uncertainty could yield overly conservative behavior and prevent the car from approaching its goal. The rationale behind this is visually depicted in Figure~\ref{fig:intuition}, showcasing why opting for a rectangular uncertainty set might lead to excessive conservatism. This phenomenon is further elaborated upon in Section~\ref{sec:experiments} within the context of a tabular model-based setting.

\begin{figure*}[ht]
\centering
\includegraphics[width=.8\textwidth]{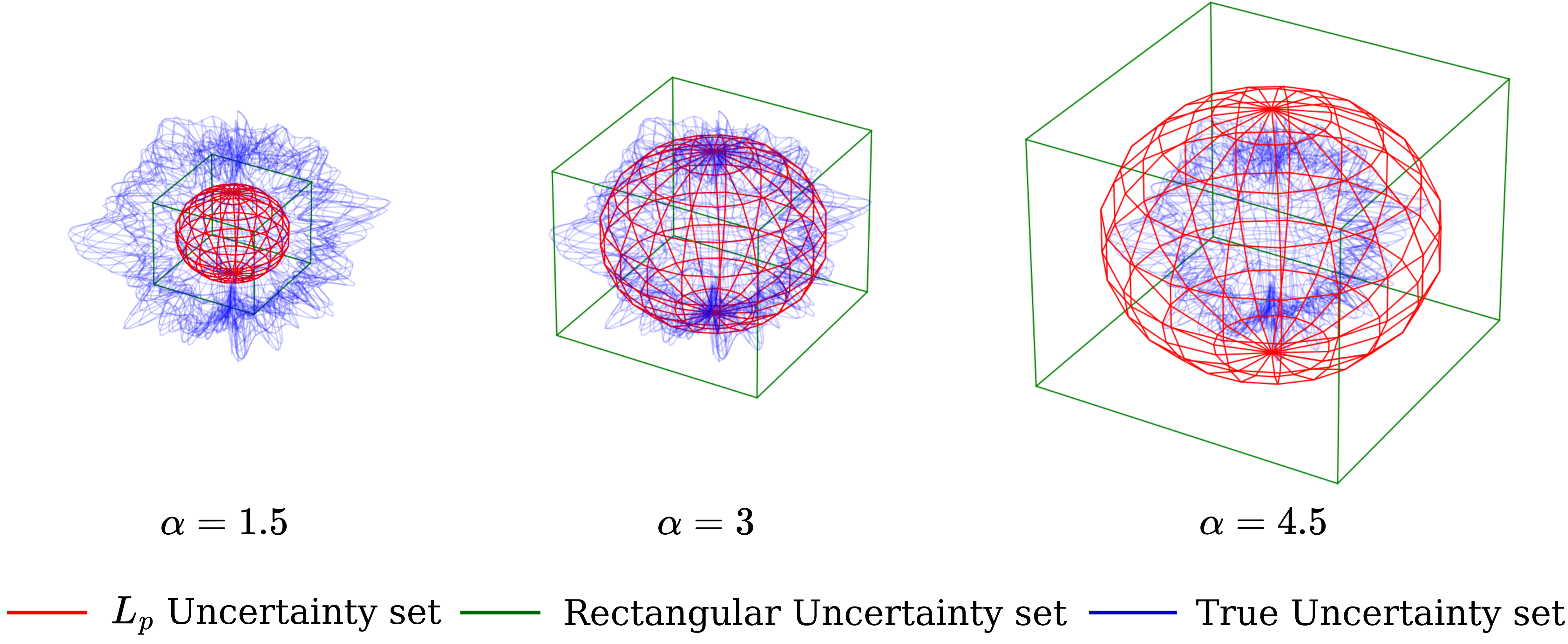}
\caption{An illustrative example of conservatism in a lower-dimensional context: When faced with an unfamiliar coupled uncertainty set (depicted in blue, see appendix for more info on this particular coupled set), we explore two potential modeling approaches. One involves an \texttt{s}-rectangular uncertainty set with a constant radius parameter $\alpha$ for each state independently (displayed in green). The other chooses a coupled uncertainty set (in red) with the same radius. By increasing $\alpha$ we are increasing conservativness. The rectangular set encompasses the actual uncertainty more swiftly. Nevertheless, this approach results in a rapid expansion of the uncertainty set to a considerable size. Conversely, the coupled set representation covers the genuine uncertainty set at a later point, yet it exhibits a lower degree of conservatism.
}
\label{fig:intuition}
\end{figure*}
In this work, we study a subclass of RMDPs where the transition model is known and the reward is uncertain but coupled. We first characterize the nice properties induced by this type of RMDP, as well as the challenges raised by reward coupling. Specifically, we show that without rectangularity, resorting to the common robust Bellman recursion leads to an incorrect and overly conservative value function. Then, under a (general) convex and compact reward uncertainty set, we establish the sufficiency of stationary policies to reach optimal robust return and prove strong duality. For reward uncertainty sets that are further specified as a norm ball centered around a nominal, we explicitly formulate the worst-case reward. The norm of interest being over the whole state-action space, the resulting set is non-rectangular. In this setting, the robust return comes out to be a regularized version of the non-robust return, where the regularization function involves the visitation frequency. This finding also enables us to: \textit{(i)} devise an efficient policy evaluation algorithm for coupled reward RMDPs; \textit{(ii)} introduce a robust policy-gradient method that trains a reward robust policy with convergence guarantees. Numerical experiments show the advantage of coupling the reward uncertainty set and illustrate the applicability of our method to high-dimensional environments. Moreover, our approach is agnostic to the reinforcement learning (RL) method being used, so it can be added on top of any learning algorithm.\\\
\textbf{Contributions. } To summarize, we make the following contributions: (1) We explicitly formulate the worst-case reward when the reward uncertainty set is a norm ball centered around a nominal, and show that it induces a regularized return whose regularizer is given by state visitation frequency; (2) We provide tractable solutions to this type of reward RMDPs and numerically test their robust behavior against relevant baselines. The proofs of all our theoretical statements can be found in the appendix at \cite{gadot2023solving}.

\section*{Related Work} 

Since the work of \citet{wiesemann2013robust}, uncertainty sets in RMDPs are commonly assumed to be $s$-rectangular, besides being convex and compact \cite{ho2018fast, ho2021partial, derman2021twice}. In fact, except for those considered in \cite{k-rectangularRMDP, goyal2023robust} which are locally coupled, $s$-rectangular uncertainty sets represent the largest class of tractable RMDPs. On the other hand, if not the studies \cite{xu2010distributionally, k-rectangularRMDP, derman2021twice, kumar2023policy} that treat both reward and transition uncertainty, RMDP literature has mostly focused just on transition uncertainty. We believe this is due to the greater challenge it represents, as the repercussions of transition ambiguity are epistemic and can lead to a butterfly effect:  a small kernel deviation at some state can have an unpredictable effect on another state so we are no longer able to track how local kernel uncertainty propagates across the state space.

Recent works have established a formal connection between reward robustness and policy regularization \cite{husain2021regularized, brekelmans2022your, eysenbach2021maximum}, while others have generalized the robustness-regularization equivalence to general RMDPs to facilitate robust RL \cite{derman2021twice, kumar2022efficient}. All these studies focused on a rectangular uncertainty set, whereas we tackle the robust problem induced by coupled reward uncertainty. This coupling leads us to derive a regularization function involving the visitation frequency, which we leverage in our policy gradient method.

In that respect, the robust policy gradient methods recently introduced in \cite{wang2022policy, kumar2023policy, li2022first} assume the uncertainty set to be rectangular. Although \citet{wang2022policy} did prove convergence in the non-rectangular case, their analysis exclusively focused on transition uncertainty while they assumed oracle access to the policy gradient. To the best of our knowledge, our work is the first to propose a provably converging policy gradient method for general reward RMDPs. 

A different line of works addresses the problem of corrupted reward signals \cite{everitt2017reinforcement, wang2020reinforcement, rakhsha2020policy, huang2020manipulating, huang2022reinforcement, nika2023online}. There, the question is how to modify the reward so that the agent is misled to a prescribed policy, but does not detect the attacking signal. In \cite{rakhsha2020policy}, the latter criterion is thought of as a budget constraint, which is formalized as the same coupled norm bound as ours. Although related to the robust setting, the two problems are complementary: a robust agent asks how to cope with an adversary while knowing its deviation level, whereas an attacker asks how to deviate the least from the observed reward so the agent is fooled and chooses a prescribed policy. Moreover, besides tackling the problem from the attacker's viewpoint, this type of study generally focuses on stealthy attacks, \ie the attacking reward value stays the same across multiple visits of the same state-action pair \cite{everitt2017reinforcement, huang2020manipulating}. In the robust setup, the agent can deal with arbitrary time-varying rewards within the uncertainty set. 

\section*{Preliminaries}
\subsection*{Notations}
For a set $\mathcal{S}$, $\lvert\mathcal{S}\rvert$ denotes its cardinal. $\langle u, v\rangle := \sum_{s\in\mathcal{S}}u_sv_s$ denotes the dot product between functions $u,v:\mathcal{S}\to\mathbb{R}$ while $\lVert v\rVert_p^q :=(\sum_{s}\lvert v(s)\rvert^p)^{\frac{1}{p}}$ is the $L_p$ norm of function $v$. For $p\in [1,\infty]$, its Hölder conjugate $q\in[1,\infty]$ is the (extended) real number such that $\frac{1}{p} + \frac{1}{q} = 1.$ Finally, we denote the probability simplex over $\St$ by $\Delta_{\St}:=\{a:\mathcal{S} \to \mathbb{R}| \sum_{s\in\mathcal{S}}a_s=1, a_s\geq 0\quad \forall s\}$.

\subsection*{Markov Decision Processes}\label{sec:MDPs}
A Markov decision process (MDP) is a tuple $(\St,\mathcal{A},P,R,\gamma,\mu)$ such that $\St$, $\mathcal{A}$ are state and action spaces respectively, $P:\St\times\mathcal{A} \to \Delta_{\St}$ is a transition kernel, $R:\St\times\mathcal{A} \to \mathbb{R}$ a reward function, $0 < \mu\in\Delta_{\St}$ an initial distribution over states and $\gamma \in [0,1)$ a discount factor ensuring that the infinite-horizon return is well-defined.  At step $t$, the agent is in some state $s_t\in\St$, executes an action $a_t$ according to a decision rule $\pi_t$ that maps past information to a probability distribution over the action space, receives a reward $R(s_t,a_t)$, and transits to another state $s_{t+1}\sim P(\cdot|s_t,a_t)$. 

A decision rule can be history-dependent or Markovian, and randomized or deterministic. A policy $\pi=(\pi_t)_{t\geq 0}$ is a sequence of decision rules whose type determines that of the policy. If the decision rules are constant over time, \ie $\pi_t=\pi_{t+1}$ for all $t\geq 0$, then the corresponding policy is said to be stationary, and we shall define it as $\pi:\St\to \Delta_{\mathcal{A}}$ with a slight abuse of notation. We further denote by $\Pi:= \Delta_{\mathcal{A}}^{\St}$ the set of all stationary policies.

Let $R^\pi(s):=\sum_{a\in\A}\pi_s(a)R(s,a)$ and $P^\pi(s'|s):=\sum_{a\in\A}\pi_s(a)P(s'|s,a), \forall s,s'\in\St,$ the expected reward and transition, respectively, where $\pi_s:=\pi(\cdot|s)$ is a shorthand notation for policy $\pi$ at state $s$. The overall goal is to maximize the following return over the policy space: 
\begin{align*}
\rho^{\pi}\rop:= \innorm{ R,d^\pi} =\innorm{\mu,v^\pi\rop},
\end{align*}     
where $d^\pi:= \mu^{\top}(\mathbf{I}_{\St}-\gamma P^\pi)^{-1}$ is the occupation measure associated with policy $\pi$
and $v^\pi\rop:= (\mathbf{I}_{\St}-\gamma P^\pi)^{-1}R^\pi$ the value function under policy $\pi$ and model parameters $(P,R)$. In this setting, it is  known that there exists a stationary policy achieving maximal return \cite{puterman2014markov}. We thus denote the optimum by $\rho^{\pi^*}\rop$. In practice, the problem can be solved through Bellman operators, respectively given by $ \mathcal{T}^\pi\rop v:= R^{\pi} + \gamma P^{\pi} v$ and $ \mathcal{T}^*\rop v:= \max_{\pi\in\Pi} \mathcal{T}^\pi\rop v, \quad\forall v\in\R^{\St}$. The subscript $R$ in the operator notation indicates the dependence on the reward function $R$, which will be useful in the reward-robust setting we introduce next.

\subsection*{Reward-Robust MDPs}
In a reward-robust MDP (reward RMDP), the reward function $R$ is unknown but lies in a given uncertainty set $\mathcal{R}$. This set is commonly assumed to be $s$-rectangular, \ie it can be decomposed over states as $\Rc = \times_{s\in\St}\Rc_s$, in which case we denote it by $\Rc^{\texttt{s}}$. If it can further be decomposed across states and actions, \ie if $\Rc = \times_{s\in\St, a\in\A}\Rc_{(s,a)}$, we will denote it by $\Rc^{\texttt{sa}}$. 

The objective is to maximize the robust performance $\rho^\pi_\Rc := \min_{R\in\Rc}\rho^\pi_{R}$ over $\Pi$. For any policy $\pi\in\Pi$, the reward model realizing the worst return is denoted by $R^\pi_{\Rc} \in \argmin_{R\in\Rc}\rho^\pi\rop$. Its corresponding robust value and robust Q-value functions are respectively defined as: 
\begin{align}
\label{eq:robust_def}
  v^\pi_{\Rc} := v^\pi_{R^\pi_{\Rc}},\quad Q^\pi_{\Rc} := Q^\pi_{R^\pi_{\Rc}}.
\end{align}
Based on non-robust definitions, they are related through: 
\[v^\pi_{\Rc}(s) = \innorm{\pi_s, Q^\pi_{\Rc}(s,\cdot)},\quad \forall  s\in\St.\]
When the uncertainty set is $s$-rectangular, the above value function coincides with the worst value, that is: 
$v^\pi_{\Rc^{\texttt{s}}} = \min_{R\in\Rc^{\texttt{s}}}v^\pi_{R}$.
On the other hand, one needs $(s,a)$-rectangularity for the same to hold for Q-values, \ie $Q^\pi_{\Rc^{\texttt{sa}}} = \min_{R\in\Rc^{\texttt{sa}}}Q^\pi_{R}$
\cite{nilim2005robust,iyengar2005robust,wiesemann2013robust,kumar2023policy}. 

The optimal robust return is defined as
$$\rho^*_\Rc := \max_{\pi\in\Pi}\rho^\pi_{\Rc}.$$
A standard way to solve RMDPs is through Bellman recursion. The robust Bellman evaluation operator is 
\[\mathcal{T}^\pi_{\Rc}v = \min_{R\in{\Rc}}\mathcal{T}^\pi_{R} v, \quad\forall v\in\R^{\St}.\]
Although non-linear, it is still a $\gamma$-contraction for any uncertainty set $\Rc$ \cite{wiesemann2013robust}. The same applies to the robust Bellman optimal operator defined as 
\[\mathcal{T}^*_{\Rc}v := \max_{\pi\in\Pi}\mathcal{T}^\pi_{\Rc}v, \quad\forall v\in\R^{\St}.\]
In the $s$-rectangular case, the robust value function $v^\pi_{\Rc^{\texttt{s}}}$ (respectively, the optimal robust value function $v^*_{\Rc^{\texttt{s}}}$) is the fixed point of the robust Bellman evaluation operator (resp., of the robust Bellman optimal operator) \cite{wiesemann2013robust}. Thus, these RMDPs can be solved using policy iteration \cite{wiesemann2013robust, ho2021partial, derman2021twice, kumar2022efficient}. 

\section*{Analyzing Reward-Robust MDPs}
\label{sec: solving reward rmdps}

In this section, we show that the above robust operators can no longer be used for general (non-rectangular) uncertainty sets $\Rc$. Indeed, as stated in Prop.~\ref{propos:non_rect_bell_operator}, the robust Bellman evaluation operator (resp., the robust Bellman optimal operator) does not admit the robust value function $v^\pi_{\Rc}$ (resp., the optimal robust value function $v^*_{\Rc^{\texttt{s}}}$) as a fixed point.

\begin{proposition} 
For non-rectangular uncertainty set $\Rc$, the robust Bellman operator $\mathcal{T}^\pi_\Rc$ (resp., $\mathcal{T}^*_\Rc$ ) has $v^\pi_{C(\Rc)}$ (resp., $v^*_{C(\Rc)}$ ) as its fixed point, where $C(\Rc)$ is the smallest $s$-rectangular uncertainty set containing $\Rc$, that is
    \[ C(\Rc) = \cap_{\Rc\subseteq\Rc^{\texttt{s}}}\Rc^{\texttt{s}}.\]
\label{propos:non_rect_bell_operator}
\end{proposition}
Hence, robust value iteration on a general (non-rectangular) uncertainty set can lead to an overly conservative solution, as $C(\Rc)$ can be much larger than $\Rc$ in large state spaces. This is illustrated in Fig. \ref{fig:intuition}.  Therefore, other methods need to be used to solve coupled reward RMDPs. Before introducing our solution, we begin by presenting key overarching findings that apply to any convex and compact reward uncertainty set.

\begin{lemma}[Stationary policies are enough] 
\label{lemma: stat policy enough}
Assume that $\Rc$ is a compact and convex set. Then, there exists a stationary policy $\pi\in \Pi$ that achieves maximal robust return:
\begin{align*}
\min_{R\in\Rc}\mathbb{E}\Bigg[\sum_{t=0}^{\infty
    }\gamma^tR(s_t,a_t)\bigm|&s_0\sim\mu, a_t\sim\pi_t(\cdot|s_t),\\
    &s_{t+1}\sim P(\cdot|s_t,a_t), \forall t\geq 0 \Bigg].
\end{align*}
\end{lemma}

The aforementioned result establishes that even though the optimal policy may be non-Markovian for general RMDPs \cite{wiesemann2013robust}, in our setting, we can focus on the set of stationary policies $\Pi$, similar to non-robust MDPs \cite{puterman2014markov}. Moreover, strong duality holds, as stated below. 

\begin{lemma}[Duality] 
\label{lemma: duality}
For all convex uncertainty sets $\Rc$, the order of optimization can be interchanged, that is
   \[\max_{\pi\in\Pi}\min_{R\in\Rc}\rho^\pi_\Rc = \min_{R\in\Rc}\max_{\pi\in\Pi}\rho^\pi_\Rc.\] 
\end{lemma}

In our framework, we examine particular constraints on reward perturbations within the aforementioned setting. Given a nominal reward denoted by $R_0 \in \mathbb{R}^{\mathcal{S} \times \mathcal{A}}$ and a positive radius $\alpha > 0$, the uncertainty set we focus on is an $L_p$-ball centered around this nominal:
\[\Rc_{p} := \{R\in\R^{\St\times\A} \mid \norm{R-R_0}_p \leq \alpha\}.\]
We note that although this constraint is restricted to $L_p$ norm balls, it is non-rectangular and still enjoys the benefit of generality. $L_p$ norm balls encompass a wide range of uncertainty patterns such as worst-case and probabilistic uncertainties, by selecting appropriate values of $p$ \cite{mannor2012lightning, delage2010percentile}.

\begin{remark}
    For the sake of simplicity and to enhance the clarity of our expression, we limit our study to $L_p$-ball constrained uncertainty sets. Nonetheless, our approach readily holds for weighted $L_p$-norms. Further elaboration on this extension can be found in the appendix.
\end{remark}

\subsection*{Worst Reward Function}

The ball structure enables us to derive the worst reward function in closed form and illuminates its effect on the occupation measure. 
This worst-case reward expression is formalized below and in fact, represents a key component of the robust learning methods introduced later on. 

\begin{theorem}[Worst-case reward]
\label{rs:rr:worstReward} 
For any policy $\pi\in\Pi$ and state-action pair $(s,a)\in\St\times \A$, the worst-case reward at $(s,a)$ is given by: 
    \[ R^\pi_{\Rc_{p}}(s,a) = R_0(s,a)-\alpha\left(\frac{d^\pi(s,a)}{\norm{ d^{\pi}}_{q}}\right)^{q-1}.\] 
    For simplicity, we will write $R^{\pi}_{p}:=R^\pi_{\Rc_{p}}$.
\end{theorem}

Thm.~\ref{rs:rr:worstReward} highlights the adversarial strategy reward-robust MDPs model. The occupation measure in the numerator shows a diminution of the reward in states that the agent frequently visits. As for the denominator, it can be thought of as the entropy of the occupancy measure: evenly distributed occupancy leads to a lower norm and a weaker adversary, whereas concentrated occupancy leads to a higher norm and a stronger adversary.  Please refer to Tab.~\ref{tb:worstReward}, for an example of the worst reward penalties for different values of $p$.

\begin{table}[ht]
    \centering
  \begin{tabular}{lll}
    \toprule                   
    $\boldsymbol{p}$     & $\boldsymbol{R^{\pi}_{p}(s,a) - R_0(s,a)}$     & \textbf{Type of penalty} \\
    \midrule
    $p$ & $\alpha\left(\frac{d^\pi(s,a)}{\norm{d^{\pi}}_q}\right)^{q-1}  $  &  General norm penalty   \\&\\
    $\infty$     & $\alpha$ & Uniform penalty \\&\\
    $2$     & $\alpha\frac{d^\pi(s,a)}{\norm{d^{\pi}}_2}$      & $\alpha$-normed frequency   \\&\\
    $1$    & $\frac{\alpha}{\abs{\X^*}}\mathds{1}\{(s,a) \in \X^*\}$     & One-hot penalty\\
    \bottomrule
  \end{tabular}
    \caption{Reward penalty induced by different coupled-reward uncertainty sets. For $p=1$, $\X^* := \argmax_{(s,a)\in\St\times\A}d^\pi(s,a)$.}
      \label{tb:worstReward}
\end{table}
Furthermore, Thm.~\ref{rs:rr:worstReward} gives us one of our main findings:
\begin{corollary}[Reward robust return]
\label{cor: reward robust return}
For a general $L_p$ norm uncertainty set, the robust return is given by:
$$\rho^\pi_{\Rc_p} = \rho^\pi_{R_0}- \alpha\norm{d^\pi}_q.$$
\end{corollary}

The factor $-\alpha\norm{d^\pi}_q$ behaves like an entropy. Indeed, it increases as the occupation measure is more distributed and vice versa.
The preceding results unveil an intriguing connection between reward-robust MDPs and regularized MDPs that employ a variant of `frequency' regularization.  This correlation mirrors earlier research efforts that explored the relationship between policy regularization and robust RL, as demonstrated in prior studies \cite{eysenbach2021maximum, derman2021twice, brekelmans2022your}. In the context of general $L_p$ norm uncertainty sets, we establish an explicit formulation for this regularizer and ascertain its reliance on the occupancy measure. Consequently, the resolution of general reward RMDPs becomes achievable by effectively addressing regularized MDPs~\cite{geist2019theory} that encompass the aspect of `frequency' regularization.
Tab. 2 in the appendix provides a comprehensive overview of the regularization function for different $L_p$-norm ball uncertainty sets. It is evident that assuming different levels of rectangularity can be likened to imposing distinct budget constraints on an adversarial entity, or `world'. In the case of $(s,a)$-rectangularity, the optimal strategy is to account for the most adverse penalty associated with each $(s,a)$ pair. On the other hand, adopting $s$-rectangularity permits the adversary to manipulate the reward function independently for each state within certain limits, thereby prompting the robust policy to distribute its visitation more evenly
across actions, yet independently for each state. This requires the potential employment of entropy-based regularization techniques. By relinquishing the constraints of rectangularity and considering a more general adversarial `budget', a robust policy would strive to distribute its visitation frequency across the entire $\St\times\A$ space, which may involve implementing a form of `frequency' regularization.


\subsection*{Policy Evaluation}
As outlined in Prop.~\ref{propos:non_rect_bell_operator}, utilizing the robust Bellman operator in the non-rectangular setting might not yield the robust value function. Nevertheless, Thm.~\ref{rs:rr:worstReward} yields the formulation of the `worst reward' Bellman operator, as articulated below.
\begin{theorem}
\label{rs:rr:rvi}
Let an uncertainty set of the form $\Rc:= \Rc_{p}$. Then, for any policy $\pi\in\Pi$, the robust value iteration
\begin{align*}
v_{n+1}(s) &=T_{R_0}^\pi v_n(s) - \alpha\frac{ \sum_{a}\pi_s(a)d^{\pi}(s,a)^{q-1}}{\norm{ d^{\pi}}_{q}^{q-1}}\\
=&:[\mathcal{T}^{\pi,\textsc{reg}}_{\Rc_{p}}v_n](s),\quad \forall  s\in\St,
\end{align*}
converges linearly to the robust value function $v^\pi_{\Rc_p}$.
\end{theorem}
This is nothing more than the non-robust Bellman operator for the MDP with the worst reward function. The new operator $\mathcal{T}^{\pi,\textsc{reg}}_{\Rc_{p}}$ preserves the $\gamma$-contracting property of the non-robust Bellman operator. Thus, the sequence given by $v_{n+1} := \mathcal{T}^{\pi,\textsc{reg}}_{\Rc_{p}} v_n$ converges to $v^\pi_{\Rc_{p}}$ (as defined in Eq.~\eqref{eq:robust_def}). 

A remaining question is how the robust Q-value $Q^\pi_{\Rc_{p}}$ relates to the robust value function $v^\pi_{\Rc_{p}}$, namely, to the fixed point of the Bellman operator introduced before. The theorem below establishes the connection between these measures. It ties the robust $Q$-function to the robust value function by the nominal non-robust Bellman operator and the `frequency' regularization term. Different expressions of this regularizer are displayed in Tab.~\ref{tb:worstReward}.

\begin{corollary}
For the uncertainty set $\Rc_\alpha$, the robust Q-value can be obtained from the robust value function via
\label{corollary:robust_Q_func}
\begin{align*}
    Q^\pi_{\Rc_{p}}(s,a) = T_{R_0}^\pi v^\pi_{\Rc_{p}}(s) -\alpha \left(\frac{d^{\pi}(s,a)}{\norm{d^{\pi}}_{q}}\right)^{q-1}.
\end{align*}
\end{corollary}

\subsubsection{Complexity Analysis}
We note that the complexity of computing an occupation measure of a given policy is \nobreak $O(S^2A\log(\frac{1}{\epsilon}))$. This implies that the complexity of policy evaluation in our algorithm is also $O(S^2A\log(\frac{1}{\epsilon}))$ for reward robust MDPs, similarly to non-robust MDPs \cite{sutton1999policy}, $(s,a)$, and $s$-rectangular robust MDPs \cite{derman2021twice,wang2022policy}. A detailed analysis can be found in the appendix. Notably, the tractability of robust policy gradient estimation for non-rectangular convex kernel uncertainty sets is still an open question.

\section*{Reward-Robust Policy Gradient}
\label{sec:lp_reward:policy_imporv}

As mentioned in Prop.~\ref{propos:non_rect_bell_operator}, employing the optimal robust Bellman operator within the non-rectangular setting may not necessarily yield the optimal robust value function. Furthermore, transforming the robust operator introduced in Thm.~\ref{rs:rr:rvi} into an optimal robust operator is not straightforward. Indeed, a greedy update in $\pi$ also impacts the `frequency' regularization component. Hence, we cannot utilize a value iteration method to achieve an optimal robust policy. Alternatively, we introduce a policy gradient method for this type of RMDPs and provide convergence guarantees.  

As a main prerequisite, we first establish a policy-gradient theorem for general reward RMDPs. 
\begin{theorem}
    \label{theorem:robust_pg}
The reward robust policy-gradient is given by: 
  \[\frac{\partial \rho^\pi_{\Rc_{p}}}{\partial \pi}  = \sum_{(s,a)\in\St\times\A}d^{\pi}(s) Q^\pi_{\Rc_{p}}(s,a)\nabla \pi_s(a),\]
  where $Q^\pi_{\Rc_{p}}$ is simply the non-robust Q-value under the worst reward, \ie $Q^\pi_{\Rc_{p}} := Q^\pi_{R^\pi_{p}}$ obtained using Cor.~\ref{corollary:robust_Q_func}.
\end{theorem}

\subsection*{Global Convergence}
We use the gradient derived in Thm.~\ref{theorem:robust_pg} to define our projected policy gradient ascent rule as
\[\pi_{k+1} := \textbf{proj}_{\Pi}\left[\pi_k + \eta_k \frac{\partial \rho^{\pi_k}_{\Rc_{p}}}{\partial \pi}\right].\]

The robust return can be non-differentiable for general uncertainty sets \cite{wang2022policy,wang2022convergence}. However, the result below establishes the differentiability of the robust return when it is constrained by an $L_p$-ball.
\begin{lemma}[Smoothness]
\label{rs:smoothness}
For all $p\in (1,\infty)$, the robust return $\rho^\pi_{\Rc_p}$ is $\beta$-smooth in $\pi$, where $\beta$ is a constant that depends on the problem parameters and is described in the appendix.
\end{lemma}
Taking step size $\eta_k =\frac{1}{\beta}$, we have the following convergence result.
\begin{theorem}[Convergence] 
\label{thm: rpg convergence}
The suboptimality gap at the $k^{th}$ iteration decays as
    \[\rho^*_{\Rc_{p}} -\rho^{\pi_k}_{\Rc_{p}}\leq c \abs{\St}\beta\frac{\rho^*_{\Rc_{p}} -\rho^{\pi_0}_{\Rc_{p}}}{k},  \]
where $c$ is a constant that depends on the discount factor $\gamma$ and on a mismatch coefficient described in the appendix.
\end{theorem}
The outcome presented here establishes the global convergence of the reward-robust policy gradient for the first time. This convergence holds with an iteration complexity of $O(\frac{1}{\epsilon})$ to attain an $\epsilon$-optimal policy, similarly to non-robust MDPs \cite{agarwal2021theory, xiao2022convergence}. It is worth noting that the aforementioned convergence result also holds for $(s,a)$ and $s$-rectangular $L_p$ constrained reward robust MDPs, maintaining the same convergence rates. Differently, under kernel uncertainty, robust policy gradient exhibits an iteration complexity of $O(\frac{1}{\epsilon^4})$ \cite{wang2022convergence}. Further comparison is detailed in the appendix.

\subsection*{Scaling Reward-Robust Policy-Gradient}
We now propose an online actor-critic algorithm that employs our method but is adaptable to high-dimensional settings (see Alg.~\ref{alg:policy_grad}). To achieve this, we utilize Thm.~\ref{rs:rr:rvi} to approximate the robust value function, a key component of Thm.~\ref{theorem:robust_pg}. Estimating the occupancy measure is also imperative for applying the `frequency'  regularizer. In this regard, we introduce the following result:
\begin{proposition}
(Lemma 1 of \cite{kumar2023policy}) For all policies $\pi$ and kernels $P$,  the iterative sequence given by  \[ d_{n+1} := \mu + \gamma P^\pi d_n, \quad\forall n\in\mathbb{N},\] converges linearly to $d^{\pi}$.
\label{prop:occupancy_mesaure_bootstrap}
\end{proposition}

\section*{Experiments}
\label{sec:experiments}
This section is dedicated to two categories of experiments. In Sec.~\ref{exp:rec_cons}, we illustrate the conservative nature of rectangularity assumptions. In Sec.~\ref{exp:pg_deep}, we assess the efficacy of our proposed algorithm in a high-dimensional setting. For reproducibility, we have provided a link to our source code in the appendix, along with comprehensive experiment details and supplementary results.

\begin{algorithm}[t]
\caption{Actor-Critic for General Reward RMDPs}
\label{alg:policy_grad}
\textbf{Input}: Differentiable policy $\pi_{\theta}(a|s)$; Q-value $Q_{\omega}(s,a)$; Frequency $d^{\pi}_{\zeta}(a|s)$, Step-sizes $\eta_{\theta}, \eta_{\omega}, \eta_{\zeta}$; Batch size $N$; Robustness radius $\alpha$ 
\begin{algorithmic}[1]
\FOR{$t=0,1,2,\cdots$}
    \STATE  Using current policy $\pi_{\theta_t}$, collect current batch $\brc*{(s_i,a_i,r_i,s_i')}_{i=1}^N$.
    \STATE Update policy parameters\\
    $\theta_{t+1} = \theta_t + \eta_{\theta} \frac{1}{N}\sum_{i=1}^N\br*{Q_{\omega}(s_i,a_i)\nabla_{\theta}\pi(a_i|s_i)}$
    \STATE Update robust Q function parameters\\ 
    $\omega_{t+1} = \omega_t + \eta_{\omega}\delta_t \nabla_{\omega}Q$, where $\delta_t= $ robust TD-error
    \STATE Compute frequency error \\
    $\Delta_t = \frac{1}{N}\sum_{i=1}^N\big(\mu(s_i)+\gamma d_{\zeta}(s_i') -d_{\zeta}(s_i)\big)$
    \STATE Update occupancy measure parameters \\
    $\zeta_{t+1} = \zeta_t + \eta_{\zeta}\Delta_t\nabla_{\zeta}d_{\zeta}$,\\
    where $\Delta_t$ is the visitation frequency error (Prop. ~\ref{prop:occupancy_mesaure_bootstrap})
\ENDFOR
\end{algorithmic}
\textbf{Output}: Robust value $Q_{\omega}$; Robust policy $\pi_{\theta}$
\end{algorithm}

\begin{figure}[t]
\centering
\includegraphics[width=0.9\linewidth]{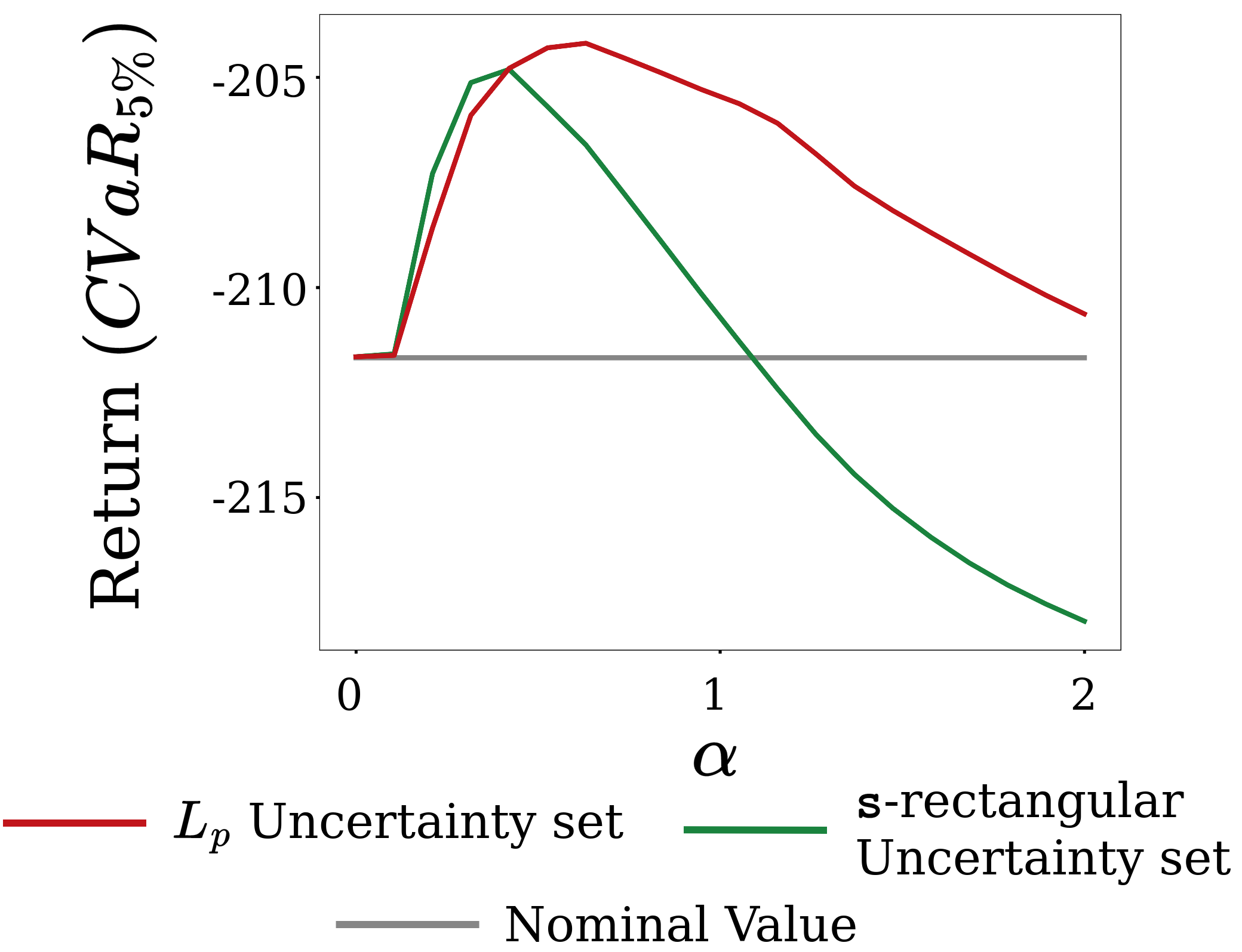}
\caption{$CVaR_{5\%}$ results for different $\alpha$}
\label{fig:tabular_pg}
\end{figure}

\begin{figure*}[ht]
\centering
\includegraphics[width=.85\textwidth]{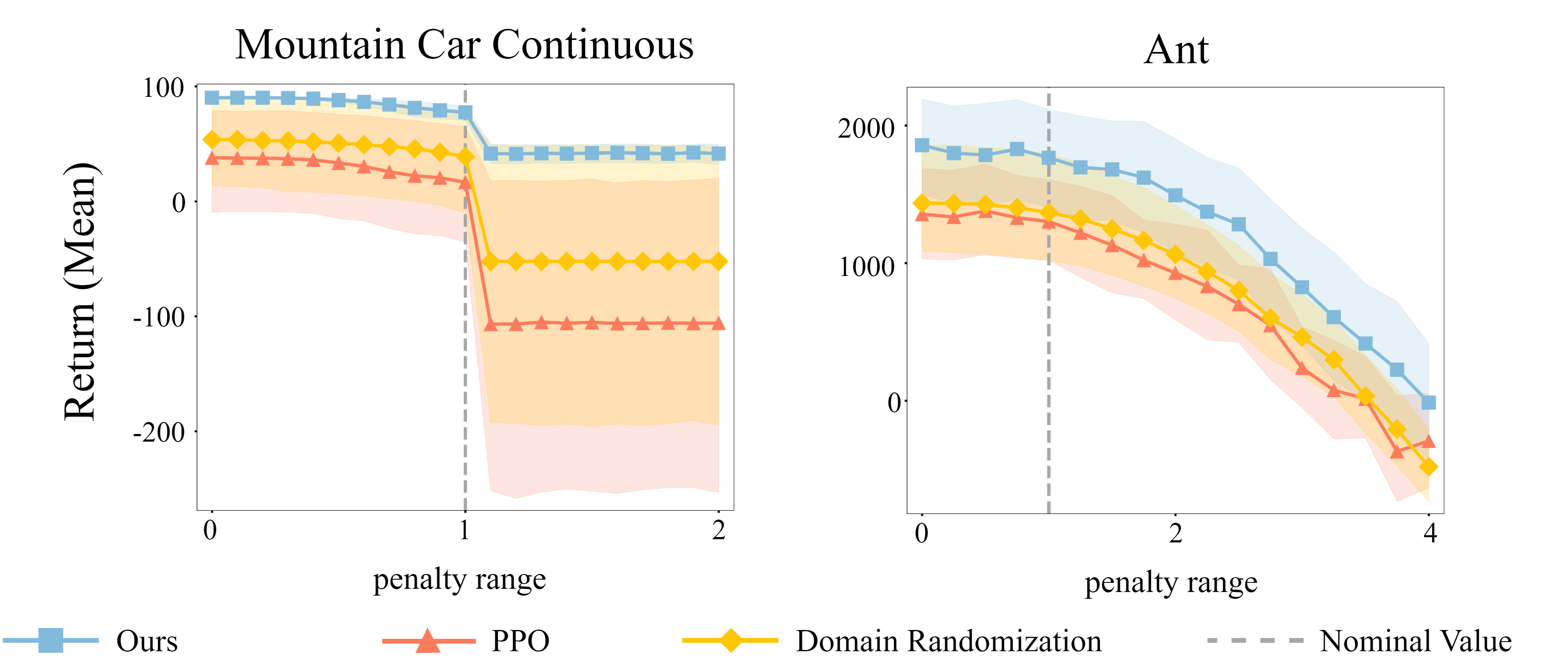}
\caption{Evaluation results on both environments for different reward perturbations.}
\label{fig:deep_results}
\end{figure*}


\subsection*{Conservative Rectangularity}
\label{exp:rec_cons}
To further illustrate the conservatism inherent to the rectangularity assumption, we explore a tabular problem. Imagine a model-based scenario where we possess knowledge of $P$, $\mu$, and $R_0\in\mathbb{R}^{\abs{\St}\times\abs{\A}}$. However, during testing, the reward function is drawn from a multivariate Gaussian distribution as follows: $R \sim \mathcal{N}(R_0, \Sigma)$, where the covariance matrix $\Sigma$ is a non-diagonal positive semi-definite matrix. It is crucial to note that the agent remains unaware of this perturbation. To derive a robust policy, we tackle this scenario using two types of uncertainty sets.

The first approach consists of treating this as an $s$-rectangular reward-RMDP with an $L_2$-norm uncertainty set, where the radius around each state remains constant, that is $\alpha_s \equiv \alpha$. The second approach adopts a coupled reward-RMDP framework with an $L_2$ norm uncertainty set, where the radius pertains to the entire reward function, labeled as $\alpha$. For both models, a soft-max parameterization is applied, and a model-based policy gradient (PG) is employed.

For the $s$-rectangular RMDP, we utilize the method described in \cite{kumar2022efficient}. In the case of the general RMDP, we employ Alg.~\ref{alg:policy_grad} in its simplified model-based version. Subsequently, we train the robust policy, subject it to testing over 1000 samples drawn from the unknown distribution, and measure the Conditional Value-at-Risk (CVaR) for the worst-performing $5\%$. This process is repeated across various $\alpha$ values. The results depicted in Figure~\ref{fig:tabular_pg} underscore that the general model attains superior `worst' performance and exhibits greater stability against radius estimation errors. This highlights that opting for a rectangular uncertainty set can significantly reduce the worst-case performance within the true uncertainty framework. For more findings from this experiment, please refer to the appendix.

\subsection*{PG For High-Dimensional Setting}
\label{exp:pg_deep}
We now undertake experiments within the online robust RL framework and evaluate the effectiveness of Alg.~\ref{alg:policy_grad} for learning robust policies. This involves training the agent using the nominal reward function and evaluating its performance under perturbed reward functions.
We examine two continuous control tasks of high dimension from OpenAI's Gym~\cite{brockman2016openai}: 'Mountain Car Continuous' \cite{Moore90efficientmemory-based} and Mujoco's \cite{todorov2012mujoco} 'Ant-v3' environment. As the baseline RL algorithm, we opt for PPO \cite{schulman2017proximal}.
In addition, to compare our method with other robust methods we consider another commonly-used robust RL approach: domain randomization.
Domain randomization trains the agent across a range of scenarios by introducing variations in the reward function during training. This equips the trained agent with robustness against analogous perturbations during testing. Notably, it is important to acknowledge that domain randomization holds an advantage over our proposed algorithm in that it can utilize multiple perturbed reward functions during training. In contrast, our algorithm remains entirely agnostic to such parameters and solely necessitates samples from the nominal reward function.
To obtain stable results, we run each experiment with 10 random seeds, and report the mean and 95\% stratified bootstrap confidence intervals (CIs)~\cite{efron1992bootstrap}.

In both environments, we introduce reward perturbation by incorporating a `penalized' segment marked by a single range parameter. Whenever the agent is within this range, it incurs a penalty proportional to its location in the area. While the range remains consistent during training, it varies during testing. The agent's performance under diverse perturbed rewards is illustrated in Figure~\ref{fig:deep_results}. Both results demonstrate that when the range significantly deviates from the nominal reward, our method outperforms the baseline PPO and the domain randomization method. While it may appear surprising that our approach exhibited superior performance compared to the non-robust algorithm under the nominal reward function, it is worth noting that previous works have shown that applying regularization may also enhance average performance~\cite{liu2019regularization}.

\section*{Conclusion And Discussion}
\label{sec:conclusion}
In this paper, we explore the often-overlooked realm of coupled RMDPs. Our attention is directed toward the context of reward uncertainty, wherein we demonstrate that the challenges posed might be less formidable than previously thought. Our study establishes that achieving tractability does not necessitate adhering to rectangularity assumptions. By drawing a direct connection between coupled $L_p$ reward RMDPs and regularized MDPs with a policy visitation frequency regularizer, we can prove the convergence of reward robust policy gradient. We present an online-based scalable algorithm for learning a robust policy within this framework and empirically substantiate our algorithm's capability to learn a robust policy.\\
Furthermore, we provide a rationale for employing a coupled uncertainty set. In the case where the uncertainty set is unknown but needs to be learned from samples \cite{lim2016reinforcement}, our coupled approach greatly facilitates learning, as it reduces the uncertainty set parameter to only one radius size. It is also more interpretable in safe RL since it can be thought of as the attacker's budget. As such, one interesting direction would be to extend our setting to the case where the uncertainty radius is unknown but needs to be inferred from trajectories.\\
One limitation of our work is that it is relevant for $L_p$-norm balls for $p>1$.
Engaging future avenues of research could involve extending the framework to accommodate an adaptive adversary or pursuing analogous outcomes within the realm of coupled kernel uncertainty RMDPs, a domain that currently remains largely unexplored.

\bibliography{main}

\newpage
\appendix
\onecolumn
\section{Connection between regularization and reward uncertainty set rectangularity}
\begin{table}[ht]
    \centering

  \begin{tabular}{lll}
    \toprule                   
    $\boldsymbol{\Rc}$     & $\boldsymbol{\rho^\pi_{R_0} - \rho^\pi_\Rc}$     & \textbf{Type of regularization} \\
    \midrule
    $\{R_0\}$ & 0 & None \\\\
    $\Rc_p^{\texttt{sa}}$     & $\sum_{s,a}d^\pi(s,a)\alpha_{(s,a)} $      & Averaged reward radius  \\\\
    $\Rc_p^{\texttt{s}}$     & $\sum_{s}d^\pi(s)\alpha_{s}\norm{\pi_s}_q $      & Averaged policy entropy  \\\\
    $\Rc_p$     & $\alpha\norm{d^\pi}_q$      & Frequency entropy  \\\\
    \bottomrule
  \end{tabular}
    \caption{Regularizers induced from different types of $L_p$ reward uncertainty sets. The subscripts $(s,a)$ and $s$ appearing in the $\alpha$-radius come from the corresponding rectangularity, allowing to choose a different radius for each state/state-action pair independently. See \cite[Theorem 1 and 2]{kumar2023policy}}
  \label{tb:returnPenalty}
\end{table} 

\section{Proofs from Sec.~\ref{sec: solving reward rmdps}: Analyzing Reward-Robust MDPs}
\subsection{Proof of Proposition \ref{propos:non_rect_bell_operator}}
\begin{proposition*} 
For non-rectangular uncertainty set $\Rc$, the robust Bellman operator $\mathcal{T}^\pi_\Rc$ (resp., $\mathcal{T}^*_\Rc$ ) has $v^\pi_{C(\Rc)}$ (resp., $v^*_{C(\Rc)}$ ) as its fixed point, where $C(\Rc)$ is the smallest $s$-rectangular uncertainty set containing $\Rc$, that is
    \[ C(\Rc) = \cap_{\Rc\subseteq\Rc^{\texttt{s}}}\Rc^{\texttt{s}}\]
\end{proposition*}

\begin{proof}
    The kernel uncertainty counterpart of the above claim, is proven in Proposition 4 of \cite{wiesemann2013robust}. Here, we prove this claim for reward case.
    
    Observe that it is enough to prove
    $$\mathcal{T}^\pi_{\Rc}v)(s) = (\mathcal{T}^\pi_{C(\Rc)}v)(s), \qquad \forall \Rc,v,\pi,s. $$
    Let $\Rc_{s} =\{R(s,\cdot )\mid R\in \Rc\}$ denotes the $s$ - th component of the uncertainty $\Rc$, further $\Rc =\times_{s}\Rc_s$. \\
    We have
    \begin{align*}
        (\mathcal{T}^\pi_{C(\Rc)}v)(s) &= \min_{R\in C(\Rc)}\sum_{a}\pi(a|s)R(s,a) + \gamma P^\pi v\\
        &=_{(i)} \min_{r_s\in C(\Rc)_s}\sum_{a}\pi(a|s)r_s(a) + \gamma P^\pi v\\
        &= \min_{r_s\in \Rc_s}\sum_{a}\pi(a|s)r_s(a) + \gamma P^\pi v,\quad \text{(by construction $\Rc_s = C(\Rc)_s$)}\\
        &= \min_{R\in \Rc}\sum_{a}\pi(a|s)R(s,a) + \gamma P^\pi v.
    \end{align*}
Where in $(i)$ we used the dependency on the $s$-th component only, where $C(\Rc) = \times_{s}C(\Rc)_s$).\\
The proofs the claim $\mathcal{T}^\pi_{\Rc}v = \mathcal{T}^\pi_{C(\Rc)}v$, which implies their fixed are the same. Further,
    \begin{align*}
      (\mathcal{T}^*_{C(\Rc)}v)(s) &= \max_{\pi} (\mathcal{T}^\pi_{C(\Rc)}v)(s),\\  
      &= \max_{\pi} (\mathcal{T}^\pi_{\Rc}v)(s),\qquad\text{(proved above)},\\ 
      &= (\mathcal{T}^*_{\Rc}v)(s),\qquad\text{(by definition)}.\\ 
    \end{align*}
    This implies $v^*_{C(\Rc)}$ is the fixed point of the both the operators $\mathcal{T}^*_\Rc$ and $\mathcal{T}^*_{C(\Rc)}$.
\end{proof}

\subsection{Proof of Lemma \ref{lemma: stat policy enough}}
\begin{lemma*}[Stationary policies are enough] Assume that $\Rc$ is a compact and convex set. Then, there exists a stationary policy $\pi\in \Pi$ that achieves maximal robust return:
\begin{align*}
\min_{R\in\Rc}\mathbb{E}\Bigm[\sum_{t=0}^{\infty
    }\gamma^tR(s_t,a_t)\bigm|&s_0\sim\mu, a_t\sim\pi_t(\cdot|s_t),
    &s_{t+1}\sim P(\cdot|s_t,a_t), \forall t\geq 0 \Bigm].
\end{align*}
\end{lemma*}

\begin{proof}
   Using the result from \cite[Theorem 5.5.1]{puterman2014markov} we know that for each history dependent policy $\pi\in\Pi_{HR}$ there exists a stationary policy $\pi'\in\Pi_{s}$ s.t.:
   \begin{equation}
   \label{res:puterman_stat}
          \mathbb{P}^{\pi}(s_t=s',a_t=a|P,s_0=s) = \mathbb{P}^{\pi'}(s_t=s',a_t=a|P,s_0=s),\quad \forall t,P,s,a,s'
   \end{equation}
   Let $\pi^*$ be an optimal policy that achieves the optimal robust return, such that:
   $$\pi^* \in \argmax_{\pi\in\Pi_{HR}}\min_{R\in\Rc}\mathbb{E}\Bigm[\sum_{n=0}^{\infty
    }\gamma^nR(s_n,a_n)|a_{n}\sim\pi_n(\cdot|s_n),P,\mu\Bigm]$$
    By \ref{res:puterman_stat} we know that there exists $\pi\in\Pi_s$ such that:
    $$\mathbb{P}^{\pi^*}(s_{t=n}=s_n,a_{t=n}=a_n|P,\mu) = \mathbb{P}^{\pi}(s_{t=n}=s_n,a_{t=n}=a_n|P,\mu)$$
    then we know that the robust return for $\pi$ is:
    \begin{align*}
        \rho^{\pi}_{\Rc} &= \min_{R\in\Rc}\mathbb{E}\Bigm[\sum_{n=0}^{\infty
    }\gamma^n\sum_{(s,a)\in\St\times\A}\mathbb{P}^{\pi}(s_{t=n}=s_n,a_{t=n}=a_n|P,\mu)R(s_n,a_n)|a_{n}\sim\pi_n(\cdot|s_n),P,\mu\Bigm]\\
    & = \min_{R\in\Rc}\mathbb{E}\Bigm[\sum_{n=0}^{\infty
    }\gamma^n\sum_{(s,a)\in\St\times\A}\mathbb{P}^{\pi^*}(s_{t=n}=s_n,a_{t=n}=a_n|P,\mu)R(s_n,a_n)|a_{n}\sim\pi_n(\cdot|s_n),P,\mu\Bigm]\\
    & = \rho^{\pi^*}_{\Rc}
    \end{align*}
    Then $\pi\in \argmax_{\pi\in\Pi_{HR}}\min_{R\in\Rc}\mathbb{E}\Bigm[\sum_{n=0}^{\infty
    }\gamma^nR(s_n,a_n)|a_{n}\sim\pi_n(\cdot|s_n),P,\mu\Bigm]$.
\end{proof}

\subsection{Proof of Lemma \ref{lemma: duality}}
\begin{lemma*}[Duality] 
For all convex uncertainty set $\Rc$, order of optimization can be interchanged, that is
   \[\max_{\pi\in\Pi}\min_{R\in\Rc}\rho^\pi_\Rc = \min_{R\in\Rc}\max_{\pi\in\Pi}\rho^\pi_\Rc.\] 
\end{lemma*}

\begin{proof}The objective can be written as 
    \begin{align*}
        \max_{\pi\in\Pi}\min_{R\in\Rc}\rho^\pi_\Rc &= \max_{\pi\in\Pi}\min_{R\in\Rc}\innorm{R,d^\pi}\\
        &= \max_{d\in\mathcal{K}}\min_{R\in\Rc}\innorm{R,d},\qquad\text{(where $\mathcal{K}:=\{d^\pi\in \R^{\St\times\A}\mid \pi\in \Pi\}$)},\\
        &= \min_{R\in\Rc}\max_{d\in\mathcal{K}}\innorm{R,d},\qquad\text{(Minimax Theorem \cite{simons1995minimax})},\\
    \end{align*}
    as $\mathcal{K},\Rc$ are compact and convex. This proves the claim.
\end{proof}

\subsection{Proof of Theorem \ref{rs:rr:worstReward}}

\begin{theorem*}[Worst-case reward]
For any policy $\pi\in\Pi$ and state-action pair $(s,a)\in\St\times \A$, the worst-case reward at $(s,a)$ is given by: 
    \[ R^\pi_{\Rc_{p}}(s,a) = R_0(s,a)-\alpha\left(\frac{d^\pi(s,a)}{\norm{ d^{\pi}}_{q}}\right)^{q-1}.\] 
    For simplicity, we will denote: $R^{\pi}_{p}:=R^\pi_{\Rc_{p}}$.
\end{theorem*}

\begin{proof}
    Firstly, for a given policy $\pi$ we know that the "worst" reward function holds the following:
    \begin{align*}     
     R^{\pi}_{p} = \argmin_{R\in\Rc}\innorm{ R,d^{\pi}}
    \end{align*}
    We can use the linearity of inner product to just look at the perturbation over the nominal reward function $R_0$
    \begin{align*}     
    R^{\pi}_{p} = R_0 + R_\alpha\\
     R_\alpha = \argmin_{\norm{R}_p\leq\alpha}\innorm{ R,d^{\pi}}
    \end{align*}
    Now by lemma \ref{lemma:holders} from helper results we know that:
    $$\forall s,a:\quad sign(R_\alpha(s,a)) = -sign(d^{\pi}_P(s,a)) = -1,\quad \abs{R_\alpha(s,a)}^p = C\cdot \abs{d^{\pi}_P(s,a)}^q$$
    To find the constant $C$ we would use our uncertainty set constraints:
    \begin{align*}
    &\norm{R_\alpha}_p = \alpha\\
    \Rightarrow&\br*{\sum_{(s,a)}|R_\alpha(s,a)|^p}^{\frac{1}{p}} = \alpha\\
    \Rightarrow&\sum_{(s,a)} C\cdot \abs{d^{\pi}_P(s,a)}^q = \alpha\\
    \Rightarrow& C = \frac{\alpha^p}{\sum_{(s,a)} \cdot \abs{d^{\pi}_P(s,a)}^q} = \frac{\alpha^p}{\norm{d^{\pi}_P}_q^q}
    \end{align*}
    Then we know that:
    \begin{align*}
        &\abs{R_\alpha(s,a)}^p = \frac{\alpha^p}{\norm{d^{\pi}_P}_q^q}\cdot \abs{d^{\pi}_P(s,a)}^q\\
        &\Rightarrow R_\alpha(s,a) = -\alpha\frac{(d^{\pi}(s,a))^{q-1}}{\lVert d^{\pi}\rVert_{q}^{q-1}}
    \end{align*}
    Which gives us the wanted result.
\end{proof}

\subsection{Proof of Corollary \ref{cor: reward robust return}}
\begin{corollary*}[Reward robust return]
For a general $L_p$ norm uncertainty set, the robust return is given by:
$$\rho^\pi_{\Rc_p} = \rho^\pi_{R_0}- \alpha\norm{d^\pi}_q.$$
\end{corollary*}

\begin{proof}
    \begin{align*}
        \rho^\pi_{\Rc_p} &= \innorm{d^\pi,R^\pi_{\Rc_p} }, \qquad \text{(by definition)}\\
        &=\sum_{s,a}d^\pi(s,a)R^\pi_{\Rc_p}(s,a)\\
        &=\sum_{s,a}d^\pi(s,a)\Bigm[R_0(s,a)-\alpha\frac{(d^{\pi}(s,a))^{q-1}}{\lVert d^{\pi}\rVert_{q}^{q-1}}\Bigm],\qquad \text{(from Theorem \ref{rs:rr:worstReward})},\\
         &=\sum_{s,a}d^\pi(s,a)R_0(s,a)-\alpha\sum_{s,a}\frac{(d^{\pi}(s,a))^{q}}{\lVert d^{\pi}\rVert_{q}^{q-1}},\\
        &=\rho^\pi_{R_0}- \alpha\norm{d^\pi}_q.\\
    \end{align*}
\end{proof}
\subsubsection{Extension to Weighted Lp norms}
Let uncertainty set be defined as 
\[\Rc_{w,p} :=\{R\mid \lVert R-R_0\rVert_{w,p}\leq \alpha\},\]
where $\lVert \cdot\rVert_{w,p}$ is weighted $L_p$ norm with weight $w$, defined as 
\[\lVert x\rVert_{w,p} = \Bigm(\sum_{s,a}w(s,a)x(s,a)^p\Bigm)^{\frac{1}{p}}.\]
\begin{corollary}[Worst-case reward]
For any policy $\pi\in\Pi$ and state-action pair $(s,a)\in\St\times \A$, the worst-case reward at $(s,a)$ is given by: 
    \[ R^\pi_{\Rc_{w,p}}(s,a) = R_0(s,a)-\alpha\left(\frac{\hat{d}^\pi(s,a)}{\norm{ \hat{d}^{\pi}}_{q}}\right)^{q-1},\] 
    where $\hat{d}^\pi(s,a) = \frac{d^\pi(s,a)}{w(s,a)^{\frac{1}{p}}}.$
\end{corollary}
\begin{proof}
    For simplicity, we will denote: $R^{\pi}_{w,p}:=R^\pi_{\Rc_{w,p}}$.
    Firstly, for a given policy $\pi$ we know that the "worst" reward function holds the following:
    \begin{align*}     
     R^{\pi}_{w,p} = \argmin_{R\in\Rc_{w,p}}\innorm{ R,d^{\pi}}
    \end{align*}
    We can use the linearity of inner product to just look at the perturbation over the nominal reward function $R_0$
    \begin{align*}     
    R^{\pi}_{w,p} &= R_0 + R_{w,\alpha}\\
     R_{w,\alpha} &= \argmin_{\norm{R}_{w,p}\leq\alpha}\innorm{ R,d^{\pi}}\\
     &= \argmin_{\norm{\hat{R}}_{p}\leq\alpha}\innorm{ \hat{R},\hat{d}^{\pi}},\qquad \text{(where $\hat{d}^\pi(s,a) = \frac{d^\pi(s,a)}{w(s,a)^{\frac{1}{p}}}, \hat{R}(s,a)=w(s,a)^{\frac{1}{p}}R(s,a)$)}.\\
    \end{align*}
Using the Holder's inequality as before, we get 
    \[ R^\pi_{\Rc_{w,p}}(s,a) = R_0(s,a)-\alpha\left(\frac{\hat{d}^\pi(s,a)}{\norm{ \hat{d}^{\pi}}_{q}}\right)^{q-1}.\] 
\end{proof}
Observe that we need to ensure that the weight $w(s,a)>0$ for all $s,a$ for which $d^\pi(s,a)\neq 0$, otherwise the adversary has infinite power. That is, the adversary can choose $R(s,a)\in (-\infty,\infty)$ for all $s,a$ which has $w(s,a) =0$, hence the robust return becomes $-\infty$, if $d^\pi(s,a)\neq 0$. Hence, to avoid pathological cases, we assume $w\succ 0$.
\begin{corollary}[Reward robust return]
For a general weighted $L_p$ norm uncertainty set, the robust return is given by:
$$\rho^\pi_{\Rc_{w,p}} = \rho^\pi_{R_0}- \alpha\norm{\hat{d}^\pi}_q,$$
where $\hat{d}^\pi(s,a) = \frac{d^\pi(s,a)}{w(s,a)^{\frac{1}{p}}}.$
\end{corollary}

\begin{proof}
Follows trivially from the worst-case reward expression derived in the result above.
\end{proof}

\subsection{Proof of Theorem \ref{rs:rr:rvi}}
\begin{theorem*}
Let an uncertainty set of the form $\Rc:= \Rc_{p}$. Then, for any policy $\pi\in\Pi$, the robust value iteration
\begin{align*}
v_{n+1}(s) &=[\mathcal{T}^{\pi,\textsc{reg}}_{\Rc_{p}}v_n](s),\quad \forall v\in\R^{\St}, s\in\St\\
&=:T_{R_0}^\pi v_n(s) - \alpha\frac{ \sum_{a}\pi_s(a)d^{\pi}(s,a)^{q-1}}{\norm{ d^{\pi}}_{q}^{q-1}}.
\end{align*}
converges linearly to robust value function $v^\pi_{\Rc_\alpha}$.
\end{theorem*}

\begin{proof}
We first fix the policy $\pi$. We know $T^\pi_R$ is $\gamma$-contraction operator for all reward function $R$ \cite{Sutton1998}. Now, we fix the reward function $R = R^{\pi}_{p}$, now we have
\begin{align}
v_{n+1}(s) &= (T^\pi_Rv_n)(s) = \sum_{a}\pi(a|s)\Bigm[R(s,a) + \gamma \sum_{s'}P(s'|s,a)v_n(s')\Bigm]\\     
&=T_{R_0}^\pi v_n(s) - \alpha\frac{ \sum_{a}\pi_s(a)d^{\pi}(s,a)^{q-1}}{\norm{ d^{\pi}}_{q}^{q-1}}.
\end{align}
We get the last equality by putting back the value of $R = R^{\pi}_{p}$ from Theorem \ref{rs:rr:worstReward}. Since, $T^\pi_R$ is $\gamma$-contraction operator, hence $v_n$ converges linearly to $v^\pi_R$ where $R = R^{\pi}_{p}$. This proves the claim.
\end{proof}

\subsection{Proof of Corollary \ref{corollary:robust_Q_func}}
\begin{corollary*}
For the uncertainty set $\Rc_\alpha$, the robust Q-value can be obtained from the robust value function as
\begin{align*}
    Q^\pi_{\Rc_{p}}(s,a) = T_{R_0}^\pi v^\pi_{\Rc_{p}}(s) -\alpha \left(\frac{(d^{\pi}(s,a))}{\norm{d^{\pi}}_{q}}\right)^{q-1}.
\end{align*}
\end{corollary*}

\begin{proof}
  Lets fix the policy $\pi$ and reward function $R = R^{\pi}_{p}$. From non-robust MDPs \cite{Sutton1998}, we have
  \begin{align*}
      Q^\pi_R(s,a) &= R(s,a)+ \gamma\sum_{s'}P(s'|s,a)v^\pi_R(s')\\
      &= R_0(s,a)-\alpha\frac{(d^{\pi}(s,a))^{q-1}}{\lVert d^{\pi}\rVert_{q}^{q-1}} + \gamma\sum_{s'} P(s'|s,a)v^\pi_R(s'),\qquad\text{(from Theorem \ref{rs:rr:worstReward})},\\
      &=T_{R_0}^\pi v^\pi_{\Rc_{p}}(s) -\alpha \left(\frac{(d^{\pi}(s,a))}{\norm{d^{\pi}}_{q}}\right)^{q-1}.
  \end{align*}
\end{proof}

\section{Proofs from Sec.~\ref{sec:lp_reward:policy_imporv}: Reward-Robust Policy Gradient}

\subsection{Proof of Theorem \ref{theorem:robust_pg}}

\begin{theorem*}
The reward robust policy-gradient is given by: 
  \[\frac{\partial \rho^\pi_{\Rc_{p}}}{\partial \pi}  = \sum_{(s,a)\in\St\times\A}d^{\pi}(s) Q^\pi_{\Rc_{p}}(s,a)\nabla \pi_s(a),\]
  where $Q^\pi_{\Rc_{p}}$ is simply the non-robust Q-value under the worst reward, \ie $Q^\pi_{\Rc_{p}} := Q^\pi_{R^\pi_{p}}$ obtained using Cor.~\ref{corollary:robust_Q_func}.
\end{theorem*}

\begin{proof}
    This follows trivially from differentiability of the robust return, established in Lemma \ref{rs:smoothness}, policy gradient theorem \cite{sutton1999policy} and envelope theorem \cite{envelopeTheorem}. 
\end{proof}

\subsection{Proof of Lemma \ref{rs:smoothness}}
First recall that a function $f:X\to \R$ is $\beta$-smooth, if 
\[\Bigm\lvert f(y)-f(x) - \innorm{\nabla f(x),y-x}\Bigm\rvert \leq \frac{\beta}{2}\lVert y-x\rVert^2_2, \qquad \forall y,x\in X.\]

The robust return can be non-differentiable for general uncertainty set \cite{wang2022policy,wang2022convergence}. However, the result below establishes the differentiablility of the reward robust return for the uncertainty sets constrained by $L_p$ norm.

\begin{lemma*}[Smoothness]
For all $p\in (1,\infty)$, the robust return $\rho^\pi_{\Rc_p}$ is $\beta$-smooth in $\pi$, where $\beta$ is a constant that depends on the system parameters.
\end{lemma*}

\begin{proof}
From Corollary \ref{cor: reward robust return}, we have
\[\rho^\pi_{\Rc_p} = \rho^\pi_{R_0}- \alpha\norm{d^\pi}_q.\]
We know $\rho^\pi_{R_0}$ is infinitely smooth in $\pi$, and its smoothness constant is derived in \cite{agarwal2021theory}. Further, $\norm{d^\pi}_q$ is infinitely-smooth in $\pi$, for all $p\in (1,\infty)$ \cite{rudin1987real}. This implies, robust return $\rho^\pi_\Rc$ is smooth in $\pi$. Since the set of all policies $\Pi$ is compact, hence the existence of smoothness constant $\beta$ is guaranteed \cite{rudin1987real}.
We know, non-robust return $\rho^\pi_{R}$ is $L$-smooth, where $L = \frac{2\gamma A}{(1-\gamma)^3}$ \cite{agarwal2021theory}. 
\end{proof}
\begin{proposition} Non-robust return $\rho^\pi$ is $L$-Lipscitz function for unit bounded reward function ($\norm{R}_\infty \leq 1$) , where $L=\frac{A}{(1-\gamma)^2}$.
\end{proposition}
\begin{proof}
    From performance difference lemma \cite{agarwal2021theory}, we  have
\begin{align}
    \rho^{\pi'} - \rho^\pi = \sum_{s,a}d^{\pi'}(s)Q^{\pi}(s,a)(\pi'(a|s)-\pi(a|s)).
\end{align}
Taking $q(s,a) = d^{\pi'}(s)Q^{\pi}(s,a) $, we have
\begin{align}
    \lVert \rho^{\pi'} - \rho^\pi\rVert_2 &= \lvert\innorm {q, \pi'-\pi}\rvert\\
    &\leq \norm {q}_2 \norm{\pi'-\pi}_2,\qquad \text{(Cauchy-Schwartz)}\\
    &\leq \norm {q}_1 \norm{\pi'-\pi}_2,\qquad \text{(using $\norm{x}_2 \leq \norm{x}_1$)}\\
    &= \sum_{s,a}d^{\pi'}(s)\abs{Q^\pi(s,a)} \norm{\pi'-\pi}_2,\qquad \text{(putting back $q$)}\\
    &\leq \sum_{s}d^{\pi'}(s)\frac{A}{1-\gamma} \norm{\pi'-\pi}_2,\qquad \text{(using  $\abs{Q(s,a)}\leq \frac{1}{1-\gamma}$)}\\
    &=\frac{A}{(1-\gamma)^2} \norm{\pi'-\pi}_2,\qquad \text{(using  $\sum_{s}d^\pi(s)= \frac{1}{1-\gamma}$)}.
\end{align}
\end{proof}

\begin{proposition} $d^{\pi}(s,a)$ is $L$-Lipschitz function in $\pi$ where $L = \frac{A}{(1-\gamma)^2}$, for all $s,a$.
\end{proposition}
\begin{proof}
It is easy to see, that 
$$ d^\pi(s,a) = \rho^{\pi}_{R_{sa}}, $$
where $R_{sa}(s',a') = \mathds{1}(s'=s,a=a')$.

From the above result, we know that the non-robust return $\rho^\pi_{R}$ is $L$-Lipschitz, where $L = \frac{A}{(1-\gamma)^2}$ \cite{agarwal2021theory}, for all $\norm{R} \leq 1$. 
 
\end{proof}
\begin{proposition} $d^{\pi}(s,a)$ is $L$-smooth in $\pi$ where $L = \frac{2\gamma A}{(1-\gamma)^3}$, for all $s,a$.
\end{proposition}
\begin{proof}
We know, non-robust return $\rho^\pi_{R}$ is $L$-smooth, where $L = \frac{2\gamma A}{(1-\gamma)^3}$ \cite{agarwal2021theory}, for all $\norm{R} \leq 1$. It is easy to see, that 
$$ d^\pi(s,a) = \rho^{\pi}_{R_{sa}}, $$
where $R_{sa}(s',a') = \mathds{1}(s'=s,a=a')$, proving the above claim.    
\end{proof}

\begin{proposition} Let $d_i\in\R^{N}$ be $L$-smooth and $K$-Lipschitz function, satisfying $\sum_{i=1}^{N}d_i = \frac{1}{1-\gamma}$ , then $\norm{d}_p$ is $2N^{\frac{p+1}{p}}(p-1)K^2 + N^{\frac{1}{p}}L$-smooth for $p\in(1,\infty)$. Furthermore, for $p=1$, it is trivial that $\norm{d}_1 = \frac{1}{1-\gamma} $ is $0$-smooth function.
\end{proposition}
\begin{proof}
    \begin{align*}
        \frac{d}{d{x}}\norm{d}_p &=\norm{d}_p^{1-p} \sum_{i}d_i^{p-1}\frac{dd_i}{dx}\\
        \implies \frac{d^2}{d{x^2}}\norm{d}_p &=\norm{d}_p^{1-p} (p-1)\sum_{i}d_i^{p-2}(\frac{dd_i}{dx})^2 + \norm{d}_p^{1-p}\sum_{i}d_i^{p-1}\frac{d^2d_i}{dx^2}\\
        &-(p-1)\norm{d}_p^{1-2p} \sum_{i,j}d_i^{p-1}d_j^{p-1}\frac{dd_i}{dx}\frac{dd_j}{dx}\\
         &=(p-1)\norm{d}_p\Bigm(\sum_{i}\frac{d_i^{p-2}}{\norm{d}_p^p}(\frac{dd_i}{dx})^2 - \sum_{i,j}\frac{d_i^{p-1}}{\norm{d}_p^p}\frac{d_j^{p-1}}{\norm{d}_p^p}\frac{dd_i}{dx}\frac{dd_j}{dx}\Bigm) + \sum_{i}\frac{d_i^{p-1}}{\norm{d}_p^{p-1}}\frac{d^2d_i}{dx^2} \\
        \implies\Bigm\lvert \frac{d^2}{d{x^2}}\norm{d}_p\Bigm\rvert &=(p-1)\norm{d}_p\Bigm(\sum_{i}\frac{d_i^{p-2}}{\norm{d}_p^p}K^2 + \sum_{i,j}\frac{d_i^{p-1}}{\norm{d}_p^p}\frac{d_j^{p-1}}{\norm{d}_p^p}K^2\Bigm) + \sum_{i}\frac{d_i^{p-1}}{\norm{d}_p^{p-1}}L \\
        &=\frac{\norm{d}^{p-2}_{p-2}}{\norm{d}_p^{p-1}} (p-1)K^2 + \frac{\norm{d}^{p-1}_{p-1}}{\norm{d}_p^{p-1}}L + (p-1)\frac{\norm{d}^{2p-2}_{p-1}}{\norm{d}_p^{2p-1}}K^2\\
        &\leq_{(i)}\frac{\norm{d}^{p-2}_{p-2}}{\norm{d}_p^{p-1}} (p-1)K^2 + N^{\frac{1}{p}}L + (p-1)\frac{\norm{d}^{2p-2}_{p-1}}{\norm{d}_p^{2p-1}}K^2\\
        &\leq_{(ii)}\frac{N^{\frac{2}{p}}}{\norm{d}_p} (p-1)K^2 + N^{\frac{1}{p}}L + (p-1)\frac{N^{\frac{2}{p}}}{\norm{d}_p}K^2\\
        &=2\frac{N^{\frac{2}{p}}}{\norm{d}_p} (p-1)K^2 + N^{\frac{1}{p}}L ,\qquad \text{}\\
        &\leq 2\frac{N^{\frac{2}{p}}}{N^{\frac{1-p}{p}}} (p-1)K^2 + N^{\frac{1}{p}}L ,\qquad \text{(using Jenson's inequality $(N\frac{1}{N}^p)^{\frac{1}{p}}\leq \norm{x}_p$)}\\
        &\leq 2N^{\frac{p+1}{p}}(p-1)K^2 + N^{\frac{1}{p}}L ,
    \end{align*}
\end{proof}
Where in $(i)$ and $(ii)$ we used - $\frac{\norm{x}_r}{\norm{x}_s}\leq dim(x)^{\frac{1}{r}-\frac{1}{s}}, $ for $r\leq s $).\\
Taking $N = SA, L = \frac{2\gamma A}{(1-\gamma)^3}$ and $K = \frac{ A}{(1-\gamma)^2} $, the above result implies the smoothness constant $\beta= O((SA)^{\frac{p+1}{p}}A^2+ (SA)^{\frac{1}{p}}A)$.

\begin{remark} Our convergence proof holds only for $p\in (1,\infty)$, not for $p=1$, due to possible non-differentiability of robust return $\rho^\pi_\Rc$. This non-differentiability significantly complicates the convergence analysis and may yield an inferior convergence rate. Hence, this analysis is left for future work.
\end{remark}

\begin{table}[ht]
    \centering
  \begin{tabular}{lll}
    \toprule                   
    Uncertainty Set     & $O$     & remark \\
    \midrule
    $\{P\}\times\Rc_{p}$ & $ (SA)^{\frac{1+2p}{p}}A\epsilon^{-1} $  &  Ours  \\&\\
    $(s,a)$-rectangular R-contamination     & $S^2A\epsilon^{-3}$ & \cite{wang2022policy} \\&\\
    Kernel Uncertainty set     & $(S^4A^4 + S^3A^5+S^2A^6)\epsilon^{-4}$      & \cite{wang2022convergence}  \\&\\
    Non Robust MDPs     &  $ SA\epsilon^{-1} $   &\cite{xiao2022convergence} \\
     \bottomrule
  \end{tabular}
    \caption{Iteration Complexity of Global Convergence of RPG  }
      \label{tb:convRate}
\end{table}

\subsection{Proof of Theorem \ref{thm: rpg convergence}}
\begin{theorem*}[Convergence] 
The suboptimality gap at the $k^{th}$ iteration decays as
    \[\rho^*_{\Rc_{p}} -\rho^{\pi_k}_{\Rc_{p}}\leq c \abs{\St}\beta\frac{\rho^*_{\Rc_{p}} -\rho^{\pi_0}_{\Rc_{p}}}{k},  \]
where $c$ is a constant that depends on the discount factor $\gamma$ and on a mismatch coefficient described in the appendix.
\end{theorem*}
\begin{proof}
  The result follows from our smoothness Lemma \ref{rs:smoothness} and combining it with convergence of smooth robust MDPs by \cite{kumar2023towards}
\end{proof}

\subsection{Proof of Proposition \ref{prop:occupancy_mesaure_bootstrap} and Complexity of Policy Evaluation}
\begin{proposition*} (Lemma 1 of \cite{kumar2023policy})
For all policy $\pi$ and kernel $P$,  the iterative sequence given by  \[ d_{n+1} := \mu + \gamma P^\pi d_n, \quad\forall n\in\mathbb{N},\] converges linearly to $d^{\pi}$.
\end{proposition*}

\begin{proof}
 We first prove,  $d^\pi \in \R^{\St} $ can be written as
 \begin{align*}
     &d^\pi = \mu^T(I-\gamma P^{\pi})^{-1} = \mu^T\sum_{n=0}^{\infty}(P^\pi)^n\\
     \implies & \gamma d^\pi P^\pi =  \Bigm(\mu^T\sum_{n=0}^{\infty}(\gamma P^\pi)^n\Bigm)\gamma P^\pi = d^\pi - I.\\
 \end{align*}
 We conclude that we have
 \[d^\pi = I+\gamma d^\pi P^\pi.\]
 Now, we have 
 \begin{align*}
     \lVert d^\pi -d^\pi_{n+1}\rVert_1 &= \lVert I+\gamma d^\pi P^\pi -\mu - \gamma d_nP^\pi\rVert_1,\qquad \text{(from definition)}\\
     &= \gamma \lVert (d^\pi -  d_n)P^\pi\rVert_1\\
     &\leq \gamma \sum_{s'}\sum_{s}\lvert d^\pi(s) -  d_n(s)\lvert P(s'|s)\\
     &= \gamma \sum_{s}\lvert d^\pi(s) -  d_n(s)\lvert \\
     &= \gamma \lVert d^\pi -d^\pi_{n}\rVert_1.
 \end{align*}
 This proves the claim. Note that convergence in not in $L_\infty$ norm but $L_1$ norm instead.
\end{proof}

Now, discuss the approximation of the robust return, formalized by the result below.
\begin{theorem} $\rho^\pi_{\Rc_p}$ can be approximated in time complexity of $O(S^2A\log(\frac{1}{\epsilon}))$ to $\epsilon$ error.
\end{theorem}
\begin{proof}
From Corollary \ref{cor: reward robust return}, we have
\[\rho^\pi_{\Rc_p} = \rho^\pi_{R_0}- \alpha\norm{d^\pi}_q.\]
Note, $\rho^\pi_{R_0}$ can approximate with $\epsilon/2$ tolerance in time complexity of $O(S^2A\log(\frac{1}{\epsilon}))$ by value iteration \cite{Sutton1998}. The iteration, 
\[d_{n+1} = \mu + d^T_nP^\pi\]
 converges linearly $d^\pi$ in $L_1$, and each iteration takes $S^2A$ time. From Proposition \ref{prop:occupancy_mesaure_bootstrap} (proof above),
$\lVert d^\pi -d^\pi_{n}\rVert_1 \leq \gamma^n \lVert d^\pi -d^\pi_{0}\rVert_1. $
 Taking $N=\log(\frac{1}{\epsilon})$, we have
 \[\lVert d^\pi -d_{N}\rVert_1  = O(\epsilon).\]
 Now, from reverse triangle inequality, we have
 \begin{align*}
     \Bigm\lvert\norm{d^\pi}_q- \norm{d_N}_q\Bigm\lvert &\leq \norm{d^\pi- d_N}_q \\
     &\leq \lVert d^\pi -d_{N}\rVert_1 \qquad\text{(from properties of $L_p$ norm)}\\
     &= O(\epsilon),\qquad\text{(proved above)}.
 \end{align*}
 This proves approximation of $\lVert d^\pi\rVert_q $ up to $\epsilon$-error requires time complexity of $O(S^2A\log(\frac{1}{\epsilon}))$. Combining both parts, we get the desired result.
\end{proof}

\subsection{Convergence of Algorithm 1}
Algorithm 1 can be proved to converge asymtotically using tools of two-time scale algorithm \cite{borkar2022stochastic}.
We can have the policy update at slow time scale whereas both Q-learning and occupation measure runs on slow time scale. Note that occupatiom measure doesn't depend on $Q$-functions, hence it runs independently. While Q-learning uses occupation measure whose estimate is getting better and better (hence the occupation measure error goes to zero), this allows Q-learning to converge at the robust Q-value.
Finally, policy iterates that runs on slow time scale, always sees the Q-value at the near converged value, hence it converges to the robust optimal policy.
We outlined the intuitive arguments, and exact analysis  follows directly from \cite{borkar2022stochastic}.

\section{Experiment details from Sec. \ref{sec:experiments}}
\subsection{Computational resources}
We used the following resources in our experiments:
\begin{itemize}
    \item \textbf{CPU:} AMD EPYC 7742 64-Core Processor
    \item \textbf{GPU:} NVIDIA GeForce RTX 2080 Ti
\end{itemize}
\subsection{Tabular experiments}
\subsubsection{Sampling the nominal model and "True" uncertainty set}
In section \ref{exp:rec_cons} we conduct a model based tabular experiment to showcase how a rectangular uncertainty set may be too much of a conservative approach. 
We chose the cardinality of $\St, \A$ and chose a seed. then we sampled a transition matrix $P$, and initial distribution $\mu$. We also sampled the nominal reward $R_0\in\mathbb{R}^{\abs{\St}\times\abs{\A}}$.
Then, for the testing part we also sampled a non-diagonal covariance matrix $\Sigma$, such that during test the reward would be $R \sim \mathcal{N}(R_0, \Sigma)$, notice that this may act as a "True" coupled uncertainty set. Even though we selected a distributional approach to the True uncertainty set, we still can treat the tail of this distribution of returns as the "worst case
 performance.
All of the sampling above were done from a uniform distribution and exact usage can be found in the code that was added to the supplementary materials.
\subsubsection{Softmax Policy Parametrization}
We parameterized our policy with a soft-max parametrization ($\pi_\theta(a|s)\propto e^{\lambda \theta(s,a)}$) such that the following holds:
Let $\pi_\theta(a|s)\propto e^{\lambda \theta(s,a)}$, then robust policy gradient is given as 
  \[\frac{\partial \rho^{\pi_\theta}_{\Rc_p}}{\partial \theta(s,a)}  = \sum_{s,a}d^{\pi_\theta}_P(s) A^{\pi_\theta}_{\Rc_p}(s,a)\pi_\theta(a|s),\]
  where $A^{\pi_\theta}_{\Rc_p}(s,a) = Q^{\pi_\theta}_{\Rc_p}(s,a) -v^{\pi_\theta}_{\Rc_p}(s)$.

\subsubsection{Experiment}
We run our policy gradient using two different approaches:
The first approach involves treating this as an $s$-rectangular reward-RMDP with an $L_2$-norm uncertainty set, where the radius around each state remains constant, denoted as $\alpha_s \equiv \alpha$. we utilize the method described in \cite{kumar2022efficient}. and trained a robust policy using the nominal model and the radius $\alpha$.

The second approach adopts a coupled reward-RMDP framework with an $L_2$-norm uncertainty set, where the radius pertains to the entire reward function, labeled as $\alpha$. We employ algorithm \ref{alg:policy_grad} in its simplified model-based version to train a robust policy for the known model and radius $\alpha$.

We then test both policies over 1000 different rewards sampled from reward distribution $R_i \sim \mathcal{N}(R_0, \Sigma)$. To check the robustness of this approach we chose to measure the Conditional Value-at-Risk (CVaR) for the worst-performing $5\%$. This was chosen in order to avoid some outliers that may affect the "worst" performance under the chosen "true" distribution.

This process is repeated across various $\alpha$ values. And for varying state sizes. The results depicted in Figure~\ref{fig:all_tabular} underscore that the general model attains superior "worst" performance and exhibits greater stability against radius estimation errors. This highlights that opting for a rectangular uncertainty set can significantly reduce the "worst-case" performance within the "True" uncertainty framework.

\begin{figure}[ht]
\centering
\includegraphics[width=1\linewidth]{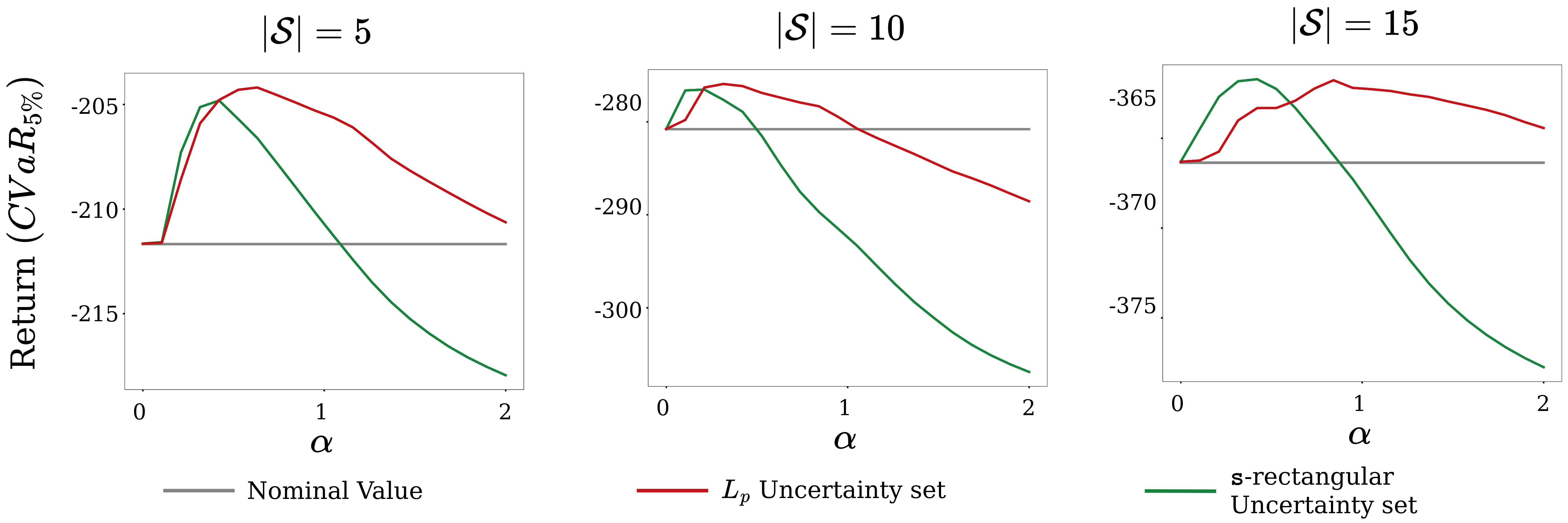}
\caption{$CVaR_{5\%}$ results for different $\alpha$, and different size of state space $\abs{\St}$}
\label{fig:all_tabular}
\end{figure}

\subsection{High-Dimensional settings}
First, to implement our experiments we chose the already existing framework for rl-baselines3-zoo \cite{rl-zoo3}, Specifically we chose their PPO implementation and didn't change any hyper-parameter they have already tuned. In tables \ref{tb:mc_hp},\ref{tb:ant_hp} we can see all of the used hyper-parameters. In section \ref{subsec:d_bootstrap} 'e elaborate on how we dealt with learning the occupancy measure.  In section \ref{subsec:environments} we give more details about the environments used and how we applied noise perturbation.
\subsubsection{Bootstrapping Occupation Measure}
\label{subsec:d_bootstrap}
We saw the computation of robust value function required the use of `occupation measure' regularization. Here, we outline how it can approximated similar to Q-value function, using bootstrapping. Prop. \ref{prop:occupancy_mesaure_bootstrap} implies that the following bootstrapping of occupation measure in sample-based regime:
\begin{align*}
    d_{n+1}(s_n) = d_n(s_n) + \eta_n[\mu(s_n) + \gamma d_n(s_{n+1}) -d_n(s_n)],
\end{align*}
where $\{s_n\}_{n\geq0}$ is the sequence of the states generated by kernel $P$ and policy $\pi$, and $\mu(s_n)=\mathds{1}\{s_n \text{ is the first state of a trajectory}\}$.
\subsubsection{Environments}
\label{subsec:environments}
For this, we chose 2 continuous control environments:\\
\textit{MountainCar-Continuous}: In this setting, a car is placed stochastically at the bottom of a sinusoidal valley, with the only possible actions being the accelerations that can be applied to the car in either direction. The goal, is to reach the right hill.\\
The state space consist of two continuous value that represent the location and the velocity of the car.
The action space consist of single continuous value that represent the amount if acceleration to apply.
The original reward function is: $r_t = -0.1 \times (action)^2$, however we added another complexity to the reward with 2 parameters, $r$ which mark the penalty range, and $\beta$ which mark the penalty scale. Then the new reward $\tilde{r}_t$ is calculated as following:
\begin{align*}
    If\ (\abs{a_t} \leq r)&\rightarrow \tilde{r}_t =r_t - \beta * (action)^2\\
    Else\ &\rightarrow \tilde{r}_t = r_t
\end{align*}
Basically this penalty motivate the agent to take larger actions so he wont suffer another penalty.
Then we trained with nominal values (depicted in table \ref{tb:reward_pertubation}) and changed the $r$ parameter during test. Figure \ref{fig:deep_results} shows the result of the testing performance.

\textit{Ant}: The ant is a 3D robot consisting of one torso (free rotational body) with four legs attached to it with each leg having two body parts. The goal is to coordinate the four legs to move in the forward (x-axis) direction by applying torques on the eight hinges connecting the two body parts of each leg and the torso.\\
The state space consist of 27 continuous value that represent positional values of different body parts of the ant, the velocities of those individual parts, and all the positions ordered before all the velocities.
The action space consist of 8 continuous value that represent torques applied at the hinge joints.
The original reward function is a complex combination of different sources: $r_t = healthy\ reward + forward\ reward - ctrl\ cost$ 
, however we added another complexity to the reward with 2 parameters, $y$ which mark the penalty range, and $\beta$ which mark the penalty scale. Then the new reward $\tilde{r}_t$ is calculated using the $y\ position$ element of the state such that:
\begin{align*}
    If\ (\abs{y \ position} \leq y)&\rightarrow \tilde{r}_t =r_t - \beta * (y\ position)^2\\
    Else\ &\rightarrow \tilde{r}_t = r_t
\end{align*}
Basically this penalty motivate the agent to step out of the range as fast as he can and then continue forward (on the x-axis).\\
Then we trained with nominal values (depicted in table \ref{tb:reward_pertubation}) and changed the $r$ parameter during test. Figure \ref{fig:deep_results} shows the result of the testing performance.
\subsubsection{Hyper-parameters}
In this section we enlist all of the used hyper-parameters in the experiments: In In table \ref{tb:mc_hp} we specify the hyper parameters the algorithms we used to train \textit{MountainCar Continuous} agent.In table \ref{tb:ant_hp} we specify the hyper parameters the algorithms we used to train \textit{Ant} agent.  (All of them are the default hyper parameters  suggested by \cite{rl-zoo3}). Table \ref{tb:reward_pertubation} enlist how we used the reward perturbations explained in \ref{subsec:environments}. Table \ref{tb:tabulr} enlist all of the hyper parameters used in the tabular experiment.
\begin{table}[h!]
    \centering
    \vspace{5pt}
    \begin{tabular}{cc}
    \toprule
    \textsc{Parameter} & \textsc{Value}  \\
    \midrule
    Seed & 1\\
    $\abs{\St}$ & $[5,10,15]$\\
    $\abs{\A}$ & $5$\\
    $\gamma$ & $0.99$\\
    learning rate & 0.01\\
       \bottomrule
    \end{tabular}
        \caption{Parameters used to test the tabular setting}
    \label{tb:tabulr}
\end{table}
\begin{table}[h!]
    \centering
    \vspace{5pt}
    \begin{tabular}{ccccc}
    \toprule
    & \textsc{Parameter} & \textsc{PPO} & \textsc{Ours}  & \textsc{Domain Randomization} \\
    \midrule
    & number of envs  & $1$ & $1$ & $1$\\
    & total timesteps  & $200000$ & $200000$ & $200000$\\
    & batch size  & $256$ & $256$ & $256$ \\
    & n steps  & $8$ &  $8$ &  $8$\\
    & gamma  & $0.9999$ & $0.9999$ & $0.9999$\\
    & learning rate  & $7.77e-05$ & $7.77e-05$  & $7.77e-05$ \\
    & entropy coefficient  & $0.00429$ & $0.00429$ & $0.00429$ \\
    & clip range  & $0.1$ & $0.1$ & $0.1$\\
    & n epochs  & $10$ & $10$  & $10$  \\
    & GAE lambda  & $0.9$ & $0.9$ & $0.9$\\
    & max grad norm  & $5$ & $5$ & $5$\\
    & vf coefficient  & $0.19$ & $0.19$ & $0.19$\\
    & use sde  & True & True & True \\
    & policy log std init  & $-3.29$ &  $-3.29$ &  $-3.29$ \\
    & Normalize  & True & True & True \\
    \midrule
    & Allow access to test perturbations  & False & False & True (Uniformly sampled) \\
    \midrule
    & Frequency function loss coefficient  & N/A & 0.5 & N/A \\
    & Robustness radius ($\alpha$)  & N/A & 0.005 & N/A \\
    \bottomrule
    \end{tabular}
        \caption{Hyper-parameters used in training for MountainCar-Continuous}
    \label{tb:mc_hp}
\end{table}

\begin{table}[h!]
    \centering
    \vspace{5pt}
    \begin{tabular}{ccccc}
    \toprule
    & \textsc{Parameter} & \textsc{PPO} & \textsc{Ours}  & \textsc{Domain Randomization} \\
    \midrule
    & number of envs  & $1$ & $1$ & $1$\\
    & total timesteps  & $1e6$ & $1e6$ & $1e6$\\
    & batch size  & $64$ & $64$ & $64$ \\
    & n steps  & $2048$ &  $2048$ &  $2048$\\
    & gamma  & $0.99$ & $0.99$ & $0.99$\\
    & learning rate  & $3e^{-4}$ & $3e^{-4}$  & $3e^{-4}$ \\
    & entropy coefficient  & $0$ & $0$ & $0$ \\
    & clip range  & $0.2$ & $0.2$ & $0.2$\\
    & n epochs  & $10$ & $10$  & $10$  \\
    & GAE lambda  & $0.95$ & $0.95$ & $0.95$\\
    & max grad norm  & $0.5$ & $0.5$ & $0.5$\\
    & vf coefficient  & $0.5$ & $0.5$ & $0.5$\\
    & use sde  & False & False & False \\
    & Normalize  & True & True & True \\
    \midrule
    & Allow access to test perturbations  & False & False & True (Uniformly sampled) \\
    \midrule
    & Frequency function loss coefficient  & N/A & 0.5 & N/A \\
    & Robustness radius ($\alpha$)  & N/A & 0.05 & N/A \\
    \bottomrule
    \end{tabular}
        \caption{Hyper-parameters used in training for Ant}
    \label{tb:ant_hp}
\end{table}

\begin{table}[h!]
    \centering
    \vspace{5pt}
    \begin{tabular}{cccc}
    \toprule
   \textsc{Env} & \textsc{Parameter} & \textsc{Train} & \textsc{Test} \\
    \midrule
     \multirow{2}{*}{\makecell{\textsc{MountainCar Continuous}}} & penalty range ($r$)  & $1$ & $[0,2]$ \\
       & penalty coefficient ($\beta$) & $0.5$ & $0.5$\\
    \midrule
     \multirow{2}{*}{\makecell{\textsc{Ant}}} & penalty range ($y$)  & $1$ & $[0,4]$\\
        & penalty coefficient ($\beta$) & $1$ & $1$\\ 
       \bottomrule
    \end{tabular}
        \caption{Reward perturbation applied during training and testing}
    \label{tb:reward_pertubation}
\end{table}


\section{Helper Results}

\begin{lemma}
\label{lemma:holders}
    Let $a^* := \argmin_{a\in\R^n}\innorm{a,b}$. Then, for all conjugate couples $p,q\in [1,\infty]$ and all coordinates $i=1,\cdots,n$, we have:
    \begin{align*}
     |a^*_i|^p &\propto |b_i|^q \\
    sign(a_i) &= -sign(b_i).
    \end{align*}
\end{lemma}

\begin{proof}
Firstly, we know that:
\begin{align*}
  &\abs{\innorm{a,b}} \leq \norm{ab}_1\\
  \Rightarrow&-\norm{ab}_1\leq\innorm{a,b}
\end{align*}
And the equality holds when:
$$sign(a_i) = -sign(b_i)$$
Now we can upper bound $\norm{ab}_1$ by using Holder's inequality:
$$\forall p,q\in[1,\infty]  s.t.: \frac{1}{p}+\frac{1}{q} = 1:\quad\norm{ab}_1 \leq \norm{a}_p\norm{b}_q$$
And we know that the equality holds when:
$$\abs{a_i}^p \propto \abs{b_i}^q.$$
\end{proof}

\end{document}